\pdfoutput=1

\documentclass{article}

\usepackage{microtype}
\usepackage{graphicx}
\usepackage{subfigure}
\usepackage{booktabs}

\usepackage{hyperref}

\usepackage[accepted]{icml2018}

\usepackage{amsmath}
\usepackage{amssymb}
\usepackage{dsfont}
\usepackage{chngpage}
\usepackage{mathtools}
\usepackage{amsthm}
\usepackage{thmtools}
\usepackage{thm-restate}

\usepackage{etoolbox}
\makeatletter
\patchcmd\@combinedblfloats{\box\@outputbox}{\unvbox\@outputbox}{}{
   \errmessage{\noexpand\@combinedblfloats could not be patched}
}
\makeatother

\declaretheorem[numberwithin=section]{lemma}
\declaretheorem[numberwithin=section]{definition}

\mathtoolsset{showonlyrefs=true}

\newcommand*{\low}{\ensuremath{\alpha}}
\newcommand*{\high}{\ensuremath{\beta}}
\newcommand*{\E}{\ensuremath{\mathbb{E}}}
\newcommand*{\Var}{\ensuremath{\mathbb{V}}}
\newcommand*{\V}{\ensuremath{\mathbb{V}}}
\newcommand*{\Cov}{\ensuremath{\mathrm{Cov}}}
\newcommand*{\States}{\ensuremath{\mathcal{S}}}
\newcommand*{\Actions}{\ensuremath{\mathcal{U}}}
\newcommand*{\R}{\ensuremath{\mathbb{R}}}

\newcommand*{\clip}{\ensuremath{\mathrm{clip}}}
\newcommand*{\pitheta}{\ensuremath{{\pi_\theta}}}
\newcommand*{\Borel}{\ensuremath{\mathcal{B}}}

\newcommand*{\I}{\ensuremath{\mathds{1}}}
\newcommand*{\capgpsi}{\ensuremath{\bar{\psi}}}
\newcommand*{\clippedpi}{\ensuremath{\bar{\pi}}}
\newcommand*{\clippedPi}{\ensuremath{\bar{\Pi}}}

\let\vec\mathbf

\let\mathbb\mathds

\renewcommand{\bar}[1]{\mkern 1.5mu\overline{\mkern-1.5mu#1\mkern-1.5mu}\mkern 1.5mu}

\icmltitlerunning{Clipped Action Policy Gradient}

\begin{document}

\twocolumn[
\icmltitle{Clipped Action Policy Gradient}

\icmlsetsymbol{equal}{*}

\begin{icmlauthorlist}
\icmlauthor{Yasuhiro Fujita}{pfn}
\icmlauthor{Shin-ichi Maeda}{pfn}
\end{icmlauthorlist}

\icmlaffiliation{pfn}{Preferred Networks, Inc., Japan}

\icmlcorrespondingauthor{Yasuhiro Fujita}{fujita@preferred.jp}
\icmlcorrespondingauthor{Shin-ichi Maeda}{ichi@preferred.jp}

\icmlkeywords{Reinforcement Learning, Policy Optimization, Policy Gradient, Continuous Control}

\vskip 0.3in
]

\printAffiliationsAndNotice{}

\begin{abstract}

Many continuous control tasks have bounded action spaces.
When policy gradient methods are applied to such tasks, out-of-bound actions need to be clipped before execution, while policies are usually optimized as if the actions are not clipped.
We propose a policy gradient estimator that exploits the knowledge of actions being clipped to reduce the variance in estimation.
We prove that our estimator, named clipped action policy gradient (CAPG), is unbiased and achieves lower variance than the conventional estimator that ignores action bounds.
Experimental results demonstrate that CAPG generally outperforms the conventional estimator, indicating that it is a better policy gradient estimator for continuous control tasks.
The source code is available at \url{https://github.com/pfnet-research/capg}.
\end{abstract}

\section{Introduction}
\label{sec:introduction}

Reinforcement learning (RL) has achieved remarkable success in recent years in a wide range of challenging tasks, such as games \cite{Mnih2015, Silver2016a, Silver2017}, robotic manipulation \cite{Levine2016}, and locomotion \cite{Schulman2015e, Schulman2017b, Heess2017}, with the help of deep neural networks.
Policy gradient methods are among the most successful model-free RL algorithms \cite{Mnih2016b, Schulman2015e, Schulman2017b, Gu2017b}.
They are particularly suitable for continuous control tasks, i.e., environments with continuous action spaces, because they directly improve policies that represent continuous distributions of actions to maximize expected returns.
For continuous control tasks, policies are typically represented by Gaussian distributions conditioned on current and past observations.

Although Gaussian policies have unbounded support, continuous control tasks often have bounded action sets that they can execute \cite{Duan2016, Brockman2016, Tassa2018}.
For example, when controlling the torques of motors, effective torque values will be physically constrained.
Policies with unbounded support like Gaussian policies are usually applied to such tasks by clipping sampled actions into their bounds \cite{Duan2016, baselines}.
Policy gradients for such policies are estimated as if actions were not clipped \cite{Chou2017}.

In this study, we demonstrate that we can improve policy gradient methods by exploiting the knowledge of actions being clipped.
We prove that the variance of policy gradient estimates can be strictly reduced under mild assumptions that hold for popular policy representations such as Gaussian policies with diagonal covariance matrices.
Our proposed algorithm, named clipped action policy gradient (CAPG), is an alternative unbiased policy gradient estimator with a lower variance than the conventional estimator.
Our experimental results on MuJoCo-simulated continuous control benchmark problems \cite{Todorov2012, Brockman2016} show that CAPG can improve the performance of existing policy gradient-based deep RL algorithms.

\section{Preliminaries}
\label{sec:preliminaries}

We consider a Markov decision process (MDP) defined by the tuple $(\States, \Actions, P, r, \rho_0, \gamma)$, where $\States$ is a set of possible states, $\Actions$ is a set of possible actions, $P$ is a state-transition probability distribution, $r: \States \times \Actions \rightarrow \R$ is a reward function, $\rho_0$ is a distribution of the initial state $s_0$, and $\gamma \in (0,1]$ is a discount factor.

A probability distribution of action conditioned on state is referred to as a policy.
The probability density function (PDF) of a policy is denoted by $\pi$.
RL algorithms aim to find a policy that maximizes the expected cumulative discounted reward from initial states,
\begin{equation}
  \eta(\pi) = \E_{s_0,u_0,\dots} \Big[ \sum_t \gamma^t r(s_t,u_t) \Big| \pi \Big],
\end{equation}
where $\E_{s_0,u_0,\dots}[\cdot|\pi]$ denotes an expected value with respect to a state-action sequence $s_0 \sim \rho_0(\cdot), u_0 \sim \pi(\cdot|s_0), s_1 \sim P(\cdot|s_0,u_0), u_1 \sim \pi(\cdot|s_1), \dots$.

The state-action value function of a policy $\pi$ is defined as
\begin{equation}
  Q^\pi(s,u) = \E_{s_1,u_1,\dots} \Big[ \sum_t \gamma^t r(s_t,u_t) \Big| s_0 = s, u_0 = u, \pi \Big].
\end{equation}

One way to find $\pi^* = \text{argmax}_\pi \eta(\pi)$ is to adjust the parameters $\theta$ of a parameterized policy $\pi_\theta$ by following the gradient $\nabla_\theta \eta(\pi_\theta)$, which is referred to a policy gradient.
The policy gradient theorem \cite{Sutton1999} states that
\begin{equation}
\label{eq:policy_gradient}
\nabla_\theta \eta(\pi_\theta) = \E_s \Big[ \E_u [Q^{\pi_\theta}(s,u) \psi(s,u)|s] \Big],
\end{equation}
where $\psi(s,u) = \nabla_\theta \log \pi_\theta(u|s)$, $\E_u[\cdot|s]$ denotes a conditional expected value with respect to $\pitheta(\cdot|s)$, and $\E_s[\cdot]$ denotes an (improper) expected value with respect to the (improper) discounted state distribution $\rho^\pitheta(\cdot)$, which is defined as
\begin{equation}
  \label{eq:improper_state_distribution}
  \rho^\pi(s) = \sum_{t} \gamma^{t}\int\rho_0(s_0)p(s_t=s|s_0,\pi)ds_0.
\end{equation}
In practice, the policy gradient is often estimated by a finite number of samples $\{(s^{(i)},u^{(i)})|u^{(i)} \sim \pi_\theta(\cdot|s^{(i)}), i=1,\dots,N\}$.
\begin{equation}
\begin{split}
  \nabla_\theta \eta(\pi_\theta) \approx \frac{1}{N} \sum_i Q^{\pi_\theta}(s^{(i)},u^{(i)}) \psi(s^{(i)},u^{(i)}).
\end{split}
\label{eq:sample_based_estimation}
\end{equation}
RL algorithms that rely on this estimation are referred to as policy gradient methods.
While this estimation is unbiased, its variance is typically high and is considered as a crucial problem of policy gradient methods.

We address the problem by estimating $\nabla_\theta \eta(\pi_\theta)$ in an unbiased and lower-variance\footnote{
When $\theta$ is not a scalar, we consider the variance of gradients with respect to each element of $\theta$ throughout the paper.
} manner than \eqref{eq:sample_based_estimation}.
To this end, we derive a random variable $Y$ such that $\V[Y] \le \V[X]$ and $\E[Y] = \E[X]$, where $X=Q^{\pi_\theta}(s,u) \psi(s,u)$.
Because $\E[X] = \E_s[\E_u[X|s]]$ and $\V[X] = \V_s[\E_u[X|s]] + \E_s[\V_u[X|s]]$, it is sufficient to show
\begin{align}
\E_u[Y|s] &= \E_u[X|s], \label{eq:e_condition}\\
\V_u[Y|s] &\le \V_u[X|s]  \label{eq:v_condition}
\end{align}
for all $s$.
For notational simplicity, $\E_u [\cdot|s]$ and $\V_u [\cdot|s]$ are written as $\E_u [\cdot]$ and $\V_u [\cdot]$ below, respectively.

The exact value of $Q^{\pi_\theta}(s,u)$ is usually not available and needs to be estimated.
It is often estimated using observed rewards after executing $u$ at $s$, sometimes combined with function approximation to balance bias and variance \cite{Schulman2016, Mnih2016b}, but this is possible only for $u$ that is executed at $s$.
Our algorithm assumes the estimates of $Q^{\pi_\theta}(s,u)$ only for such $(s,u)$ pairs to be available, and thus is applicable to such cases.

\section{Clipped Action Policy Gradient}

We consider the case where any action $u \in \R^d$ ($d \ge 1$) chosen by an agent is clipped by the environment into a range $[\low, \high] \subset \R^d$.
That is, the state-transition PDF and the reward function satisfy
\begin{align}
  P(s' | s, u) &= P(s' | s,\clip(u, \low, \high)), \label{eq:clipped_transition} \\
  r(s, u)     &= r(s, \clip(u, \low, \high)), \label{eq:clipped_reward}
\end{align}
respectively.
The $\clip$ function is defined as $\clip(u, \low, \high) = \max(\min(u, \high), \low)$, where $\max$ and $\min$ are computed elementwise when $u$ is a vector, i.e., $d \ge 2$.
Each of $\low$ and $\high$ can be a constant or a function of $s$.
The case where the reward function depends on actions before clipping is discussed in Section~\ref{subsec:extensions}.

Before explaining our algorithm, let us characterize the class of policies we consider in this study.
\begin{definition}[compatible PDF]
Let $p_{\theta}(u)$ be a PDF of $u \in \R$ that has a parameter $\theta$.
If $p_{\theta}(u)$ is differentiable with respect to $\theta$ and allows the exchange of derivative and integral as
$\int_{-\infty}^\low \nabla_\theta p_\theta(u)du = \nabla_\theta \int_{-\infty}^\low p_\theta(u)du$ and
$\int_{\high}^\infty \nabla_\theta p_\theta(u)du = \nabla_\theta \int_{\high}^\infty p_\theta(u)du$, we call $p_\theta(u)$ a compatible PDF.
If $p_{\theta}(u|s)$ is a conditional PDF that satisfies these conditions, we call it a compatible conditional PDF.
\end{definition}

\subsection{Scalar actions}
\label{subsec:scalar_action}

First, we derive an unbiased and lower-variance estimator of the policy gradient for scalar actions, i.e., $d = 1$.
The case of vector actions will be covered later in Section~\ref{subsec:vector_action}.

From \eqref{eq:clipped_transition} and \eqref{eq:clipped_reward}, the state-action value function satisfies
\begin{align}
Q^{\pi_\theta}(s,u)
&= Q^{\pi_\theta}(s,\clip(u,\low,\high)) \\
&=
\begin{cases}
    Q^{\pi_\theta}(s,\low) & \text{if } u \le \low\\
    Q^{\pi_\theta}(s,u)    & \text{if } \low < u < \high\\
    Q^{\pi_\theta}(s,\high)& \text{if } \high \le u
\end{cases}.
\label{eq:disjoint_q}
\end{align}

Let $X$ be a random variable that depends on $u$ and $\I_{f(u)}$ be an indicator function that takes 1 when $u$ satisfies the condition $f(u)$, otherwise 0.
Because $X = \I_{u \le \low} X + \I_{\low < u < \high} X + \I_{\high \le u} X$, $\E_u[X]$ can be decomposed as
\begin{equation}
 \E_u [X] = \E_u [\I_{u \le \low} X] + \E_u [\I_{\low < u < \high} X] + \E_u [\I_{\high \le u} X].
\label{eq:disjoint_mean}
\end{equation}
From \eqref{eq:disjoint_q} and \eqref{eq:disjoint_mean}, we have
\begin{align}
\MoveEqLeft
\E_u[Q^{\pi_\theta}(s,u) \psi(s,u)] \\
&=
\begin{aligned}[t]
&Q^{\pi_\theta}(s,\low) \E_u [\I_{u \le \low}  \nabla_\theta \log \pi_\theta(u|s)] \\
&+\E_u [\I_{\low < u < \high} Q^{\pi_\theta}(s,u) \nabla_\theta \log \pi_\theta(u|s)] \\
&+Q^{\pi_\theta}(s,\high) \E_u [\I_{\high \le u} \nabla_\theta \log \pi_\theta(u|s)].
\end{aligned}
\label{eq:original_disjoint_estimator}
\end{align}

Meanwhile, the following useful lemma holds.
\begin{restatable}{lemma}{lemmaestimator}
\label{lemma:estimator}
Suppose $\pi_\theta(u|s)$ is a compatible conditional PDF of $u \in \R$ whose cumulative distribution function (CDF) is $\Pi_\theta(u|s)$.
Then, the following equations hold:
\begin{align}
\E_u[\I_{u \le \low} \nabla_\theta \log \pi_\theta(u|s)]  &= \E_u [\I_{u \le \low} \nabla_\theta \log \Pi_\theta(\low|s)],\\
\E_u[\I_{\high \le u} \nabla_\theta \log \pi_\theta(u|s)] &= \E_u [\I_{\high \le u} \nabla_\theta \log (1 - \Pi_\theta(\high|s))].
\end{align}
\end{restatable}
See the appendix for the proof.

By applying Lemma~\ref{lemma:estimator} to \eqref{eq:original_disjoint_estimator},
we can construct an alternative estimator
\begin{align}
\MoveEqLeft
\E_u[Q^{\pi_\theta}(s,u) \psi(s,u)]\\
&=
\begin{aligned}[t]
  &Q^{\pi_\theta}(s,\low) \E_u [\I_{u \le \low} \nabla_\theta \log \Pi_\theta(\low|s)]\\
  &{}+\E_u [\I_{\low < u < \high} Q^{\pi_\theta}(s,u)\nabla_\theta \log \pi_\theta(u|s)]\\
  &{}+Q^{\pi_\theta}(s,\high) \E_u [\I_{\high \le u}\nabla_\theta \log \left(1- \Pi_\theta(\high|s)\right)]\\
\end{aligned}\\
&=\E_u [Q^{\pi_\theta}(s,u) \capgpsi(s,u)],
\label{eq:proposed_disjoint_estimator}
\end{align}
where
\begin{equation}
\capgpsi(s,u) =
\begin{cases}
    \nabla_\theta \log \Pi_\theta(\low|s)      & \text{if } u \le \low\\
    \nabla_\theta \log \pi_\theta(u|s)         & \text{if } \low < u < \high\\
    \nabla_\theta \log (1-\Pi_\theta(\high|s)) & \text{if } \high \le u
\end{cases}.
\label{eq:bar_psi}
\end{equation}

By \eqref{eq:proposed_disjoint_estimator} the policy gradient can be estimated using the sample average of $Q^{\pi_\theta}(s,u) \capgpsi(s,u)$.
This estimator, which we call clipped action policy gradient (CAPG), is better than the conventional estimator \eqref{eq:sample_based_estimation} in the sense that it has a lower variance while being unbiased.

The difference between the conventional estimator and CAPG comes from outside the action bounds.
CAPG replaces $\pi_\theta(u|s)$ of $\nabla_\theta \log \pi_\theta(u|s)$ with $\Pi_\theta(\low|s)$ and $1-\Pi_\theta(\high|s)$ at $u\le\low$ and $\high\le u$, respectively.
Intuitively speaking, because both $\Pi_\theta(\low|s)$ and $1-\Pi_\theta(\high|s)$ are deterministic given $s$, the variance should decrease.
In fact, this observation is true.

To show this, we need to decompose the variance.
The variance of a random variable $X$ that depends on $u$ can be decomposed as
\begin{align}
\V_u[X] ={} &\V_u [\I_{u \le \low} X]\! + \!\V_u [\I_{\low < u < \high} X]\! + \!\V_u [\I_{\high \le u} X]\\
 &{}-2\E_u [\I_{u \le \low} X] \E_u [\I_{\low < u < \high} X]\\
 &{}-2\E_u [\I_{\low < u < \high} X] \E_u [\I_{\high \le u} X]\\
 &{}-2\E_u [\I_{\high \le u} X] \E_u [\I_{u \le \low} X].
\label{eq:disjoint_variance}
\end{align}
Let us compare each term of the right-hand side between the cases $X=Q^{\pi_\theta}(s,u)\psi(s,u)$ and $X=Q^{\pi_\theta}(s,u) \capgpsi(s,u)$.
From Lemma~\ref{lemma:estimator}, we can see that the terms $\V_u [\I_{\low < u < \high} X]$, $\E_u [\I_{u \le \low} X], \E_u [\I_{\low < u < \high} X]$, and $\E_u [\I_{\high \le u} X] $ do not make any differences.
The following lemma shows that the difference arises only from the terms $\V_u [\I_{u \le \low} X]$ and $\V_u [\I_{\high \le u} X]$.
\begin{restatable}{lemma}{lemmavarianceinequality}
\label{lemma:variance_inequality}
Suppose $\pi_\theta(u|s)$ is a compatible conditional PDF of $u \in \R$ whose CDF is $\Pi_\theta(u|s)$.
Then, the following inequalities hold:
\begin{align}
\MoveEqLeft
\Var_u[\I_{u \le \low} \nabla_\theta \log \pi_\theta(u|s)]
\geq \Var_u[\I_{u \le \low} \nabla_\theta \log \Pi_\theta(\low|s)],
\\
\MoveEqLeft
\Var_u[\I_{\high \le u} \nabla_\theta \log \pi_\theta(u|s)]
\geq \Var_u[\I_{\high \le u} \nabla_\theta \log (1-\Pi_\theta(\high|s))].
\end{align}
The equalities hold only when $\nabla_\theta \log \pi_\theta(u|s)$ is constant over $u \le \low$ and $\high \le u$, respectively.
\end{restatable}
See the appendix for the proof.

Combining Lemma~\ref{lemma:estimator} and Lemma~\ref{lemma:variance_inequality} leads to the following result.
\begin{lemma}
\label{lemma:CAPG}
Suppose $\pi_\theta(u|s)$ is a compatible conditional PDF of $u \in \R$ whose CDF is $\Pi_\theta(u|s)$.
Let $f(s,u)$ be a real-valued function such that
\begin{equation}
f(s,u) =
\begin{cases}
    f(s,\low) & \text{if } u \le \low\\
    f(s,u)    & \text{if } \low < u < \high\\
    f(s,\high)& \text{if } \high \le u
\end{cases}.
\end{equation}
Define $\psi(s,u) = \nabla_\theta \log \pi_\theta(u|s)$ and $\capgpsi(s,u)$ as \eqref{eq:bar_psi}.
Then, the following equality and inequality hold:
\begin{align}
  \E_u[f(s,u)\capgpsi(s,u)] =   \E_u[f(s,u)\psi(s,u)],\\
  \V_u[f(s,u)\capgpsi(s,u)] \le \V_u[f(s,u)\psi(s,u)].
\end{align}
The equality of the variances holds only when $\psi(s,u)$ is constant over both $u \le \low$ and $\high \le u$.
\end{lemma}
Lemma~\ref{lemma:CAPG} shows that both \eqref{eq:e_condition} and \eqref{eq:v_condition} are satisfied when $Y=Q^{\pi_\theta}(s,u) \capgpsi(s,u)$ and $X=Q^{\pi_\theta}(s,u) \psi(s,u)$.
Therefore, we can conclude that CAPG has a lower variance than the conventional estimator while being unbiased.

\subsection{Vector actions}
\label{subsec:vector_action}

The results in the previous subsection can be extended to the case of vector actions, $\vec{u} \in \R^d$ where $d \ge 2$, as long as the elements of $\vec{u}$ are conditionally independent given $s$, i.e., the PDF can be factored as
\begin{equation}
  \pi(\vec{u}|s) = \pi_\theta^{(1)}(u_1|s) \pi_\theta^{(2)}(u_2|s) \cdots\pi_\theta^{(d)}(u_d|s),
  \label{eq:factor_pdf}
\end{equation}
where $u_i$ denotes the $i$-th element of $\vec{u}$, and $\pi^{(i)}_\theta$ denotes its corresponding conditional PDF.
A typical example of such a policy is a multivariate Gaussian policy with a diagonal covariance.

\begin{restatable}{lemma}{lemmavectoraction}
\label{lemma:vector_action}
Suppose $\pi_\theta(\vec{u}|s)$ is a conditional PDF of $\vec{u} \in \R^d$ ($d \ge 2$) whose CDF is $\Pi_\theta(\vec{u}|s)$.
The conditional PDF and CDF of $u_i$ are denoted by $\pi^{(i)}_\theta$ and $\Pi^{(i)}_\theta$, respectively.
Suppose each $\pi^{(i)}_\theta$ is compatible and the elements of $\vec{u}$ are conditionally independent given $s$.
Let $f(s,\vec{u})$ be a real-valued function such that $f(s,\vec{u}) = f(s,\clip(\vec{u},\low,\high))$.
Define $\psi(s,\vec{u}) = \sum_i \psi^{(i)}(s,u_i)$, where $\psi^{(i)}(s,u) = \nabla_\theta \log \pi^{(i)}_\theta(u|s)$.
Similarly, define $\capgpsi(s,\vec{u}) = \sum_i \capgpsi^{(i)}(s,u_i)$, where
\begin{equation}
\capgpsi^{(i)}(s,u) =
\begin{cases}
    \nabla_\theta \log \Pi_\theta^{(i)}(\low|s)      & \text{if } u \le \low\\
    \nabla_\theta \log \pi_\theta^{(i)}(u|s)         & \text{if } \low < u < \high\\
    \nabla_\theta \log (1-\Pi_\theta^{(i)}(\high|s)) & \text{if } \high \le u
\end{cases}.
\end{equation}
Then, the following equality and inequality hold:
\begin{align}
  \E_{\vec{u}}[f(s,\vec{u})\capgpsi(s,\vec{u})] =   \E_{\vec{u}}[f(s,\vec{u})\psi(s,\vec{u})], \label{eq:vector_capg_equality} \\
  \V_{\vec{u}}[f(s,\vec{u})\capgpsi(s,\vec{u})] \le \V_{\vec{u}}[f(s,\vec{u})\psi(s,\vec{u})]. \label{eq:vector_capg_inequality}
\end{align}
The equality of the variances holds only when $\psi^{(i)}(s,u)$ is constant over both $u \le \low$ and $\high \le u$ for all $1 \le i \le d$.
\end{restatable}
See the appendix for the proof.

\subsection{Implementation}
\label{subsec:implementation}

CAPG can be easily incorporated into existing policy gradient-based algorithms.
We only have to replace the computation of $\psi(s,u)$ with that of $\capgpsi(s,u)$ to use CAPG.
When $\psi(s,u)$ is computed using an automatic differentiation tool, we can instead replace $\log \pi_\theta(u|s)$ with
\begin{equation}
\begin{aligned}
    \log \Pi_\theta(\low|s)      && &\text{if } u \le \low\\
    \log \pi_\theta(u|s)         && &\text{if } \low < u < \high\\
    \log (1-\Pi_\theta(\high|s)) && &\text{if } \high \le u
\end{aligned}.
\end{equation}

\subsection{Extensions}
\label{subsec:extensions}

Although we have used standard notations of MDPs, our results do not rely on the Markov property.
CAPG works as an unbiased and lower-variance policy gradient estimator in non-Markovian environments as well, in the same way that the REINFORCE algorithm \cite{Williams1992} works in such environments.

We assumed \eqref{eq:clipped_reward} so that $Q^\pitheta(s,u)$ becomes constant outside the action bounds.
However, sometimes it makes sense to use a reward function that depends on out-of-bound actions even when the state-transition dynamics does not, e.g., to penalize the norm of actions to prevent the policy from going too far out of the bounds.
With such a reward function, \eqref{eq:disjoint_q} no longer holds.
Instead, we can use the recursive structure of $Q^{\pi_\theta}(s,u)$ to obtain
\begin{align}
\MoveEqLeft
\E_u[Q^{\pi_\theta}(s,u)\psi(s,u)]\\
&=
\begin{aligned}[t]
&\E_u[r(s,u) \psi(s,u)]\\
&+\E_u[\gamma \E_{s',u'}[Q^{\pi_\theta}(s',u')] \psi(s,u)]],
\end{aligned}
\label{eq:separate_reward_from_q}
\end{align}
where $\E_{s',u'}[\cdot]$ denotes an expected value with respect to $s' \sim P(s,\clip(u,\low,\high),\cdot)$, $u' \sim \pitheta(\cdot|s')$.
We can apply CAPG to the second term of the right-hand side of \eqref{eq:separate_reward_from_q} because $\gamma \E_{s',u'}[Q^{\pi_\theta}(s',u')]$ only depends on $u$ via $\clip(\cdot,\low,\high)$.

\subsection{Clipped distribution}
\label{subsec:clipped_distribution}

So far we have derived CAPG as a better policy gradient estimator.
We now argue that CAPG can be interpreted as estimating the policy gradient of a transformed policy.

Given a policy $\pi_\theta$ and action bounds $[\low,\high]$, we can consider a policy $\clippedpi_\theta$ modeled as a probability distribution with bounded support whose CDF is defined as
$\clippedPi_\theta(u|s) = \I_{\low \le u < \high}\Pi_\theta(u|s) + \I_{\high \le u}$,
which is a mixture of two degenerate distributions at \{$\low$, $\high$\} and a truncated version of $\pitheta$.
The corresponding PDF with respect to the measure generated by the mixture
\footnote{
The probability measure $P$ corresponding to $\clippedPi_\theta(u|s)$, defined over the measurable space $([\low,\high], \Borel([\low,\high]))$, is such that $P \ll \lambda + \delta_\low + \delta_\high$, where $\Borel$ is the Borel $\sigma$-algebra, $\lambda$ is the Lebesgue measure, and $\delta_x$ is a Dirac measure at $x$.
}
is given by
\begin{equation}
\clippedpi_\theta(u|s) =
\begin{cases}
    \Pi_\theta(\low|s)  & \text{if } u = \low \\
    \pi_\theta(u|s)     & \text{if } \low < u < \high \\
    1 - \Pi_\theta(\high|s) & \text{if } u = \high
\end{cases}.
\label{eq:pi_c_pdf}
\end{equation}
We call this distribution a clipped distribution.
Seeing that $\nabla_\theta \log \clippedpi_\theta(u|s) = \capgpsi(s,u)$ for $u \in [\low,\high]$, CAPG applied to $\pi_\theta$ is, in fact, estimating the policy gradient of $\clippedpi_\theta$.
If we see Gaussian policies used with action bounds as clipped Gaussian policies, then CAPG is the straightforward policy gradient estimator for them, whereas the conventional estimator has an unnecessarily high variance.

While a clipped distribution resembles a truncated distribution, they are different.
A clipped distribution can be multimodal even when its underlying distribution is unimodal because it puts the probability mass at the action bounds.
In contrast, a truncated distribution is always unimodal when its underlying distribution is unimodal.
This makes a difference in their representational powers to model policies.

\section{Experiments}

In this section, we evaluate the performance of CAPG compared to the conventional policy gradient estimator, which we call PG, in problems with action bounds.

\subsection{Continuum-armed bandit problems}

\begin{figure*}[!t]
  \centering
  \includegraphics[width=0.24\hsize]{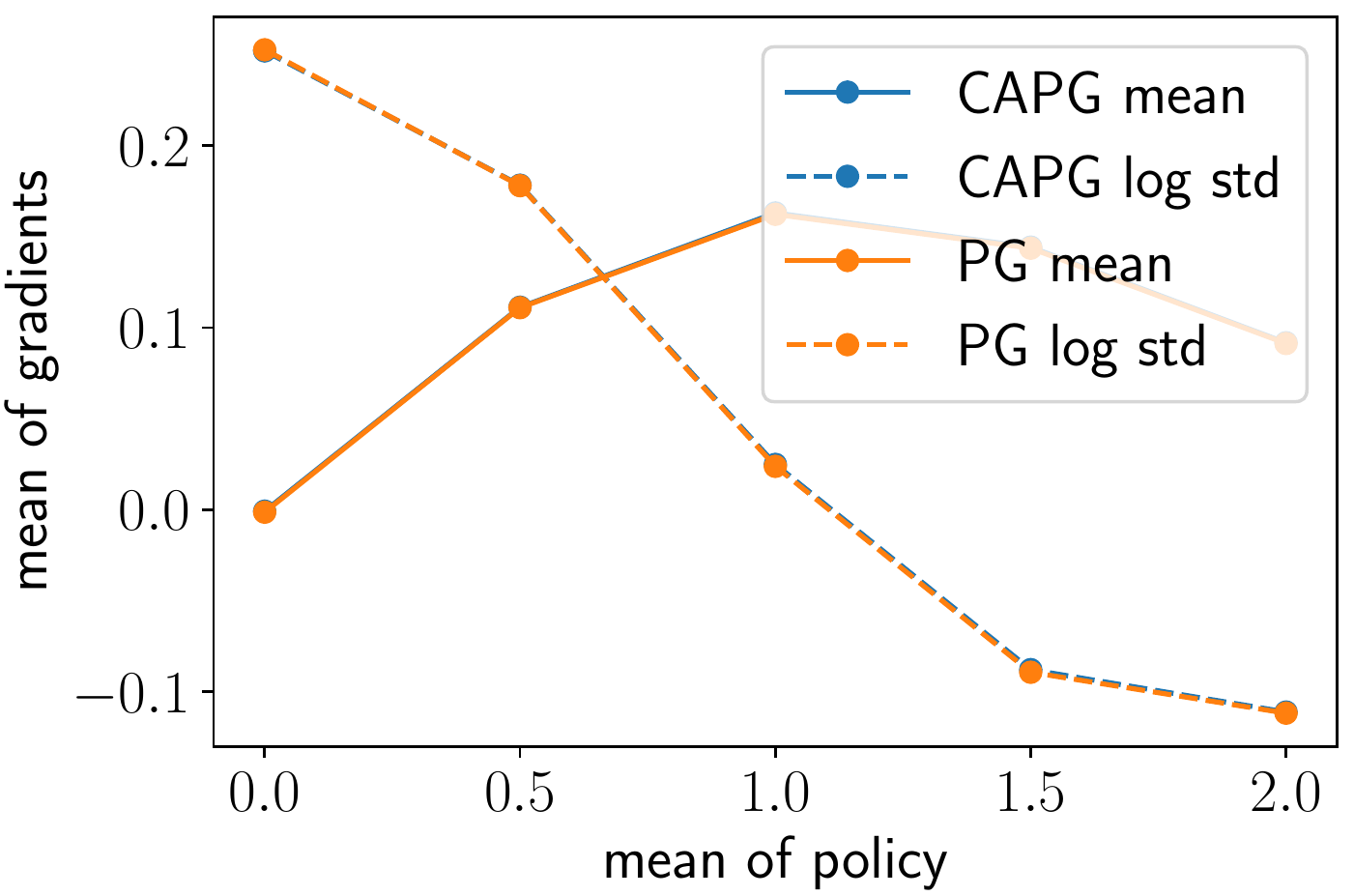}
  \includegraphics[width=0.24\hsize]{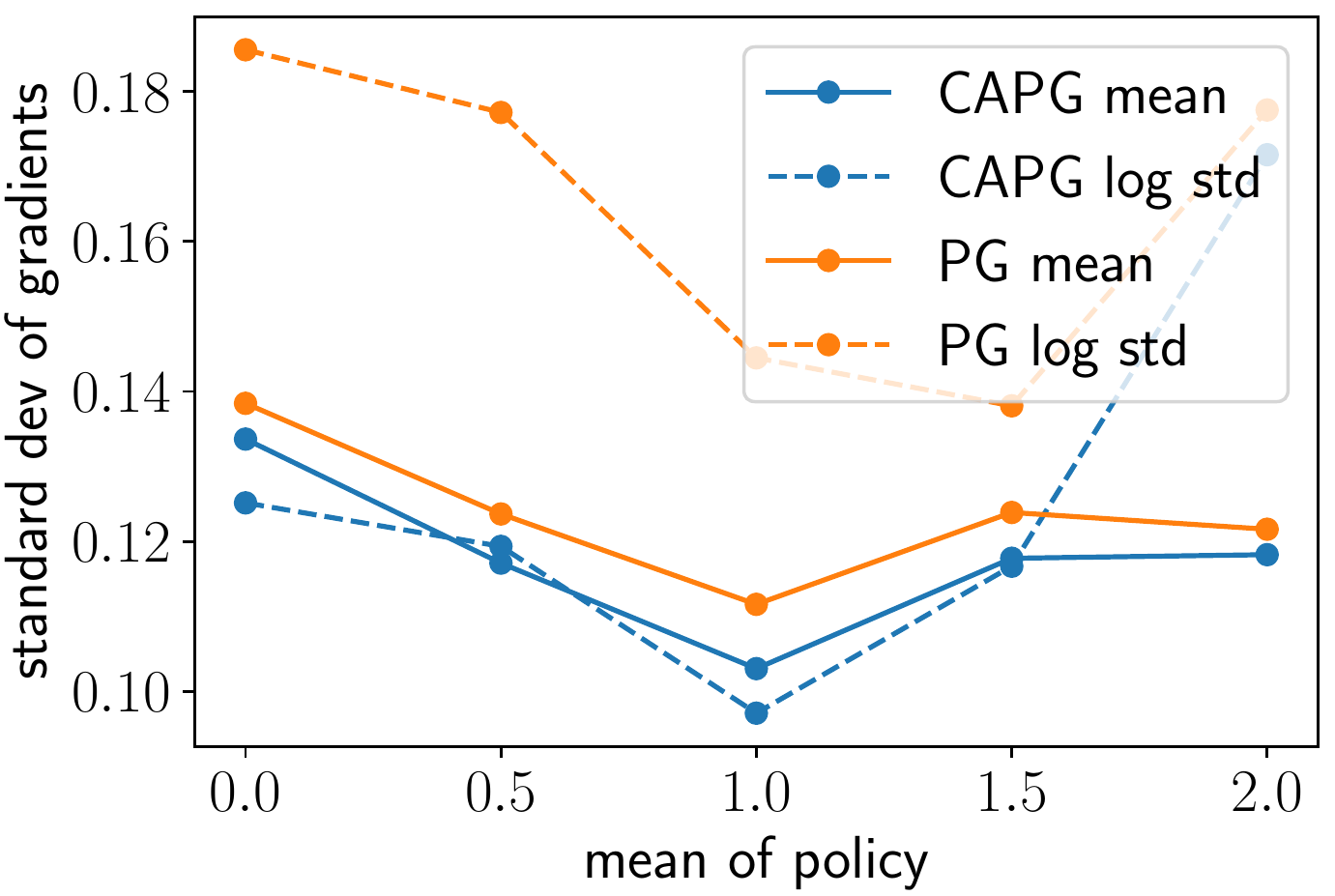}
  \includegraphics[width=0.24\hsize]{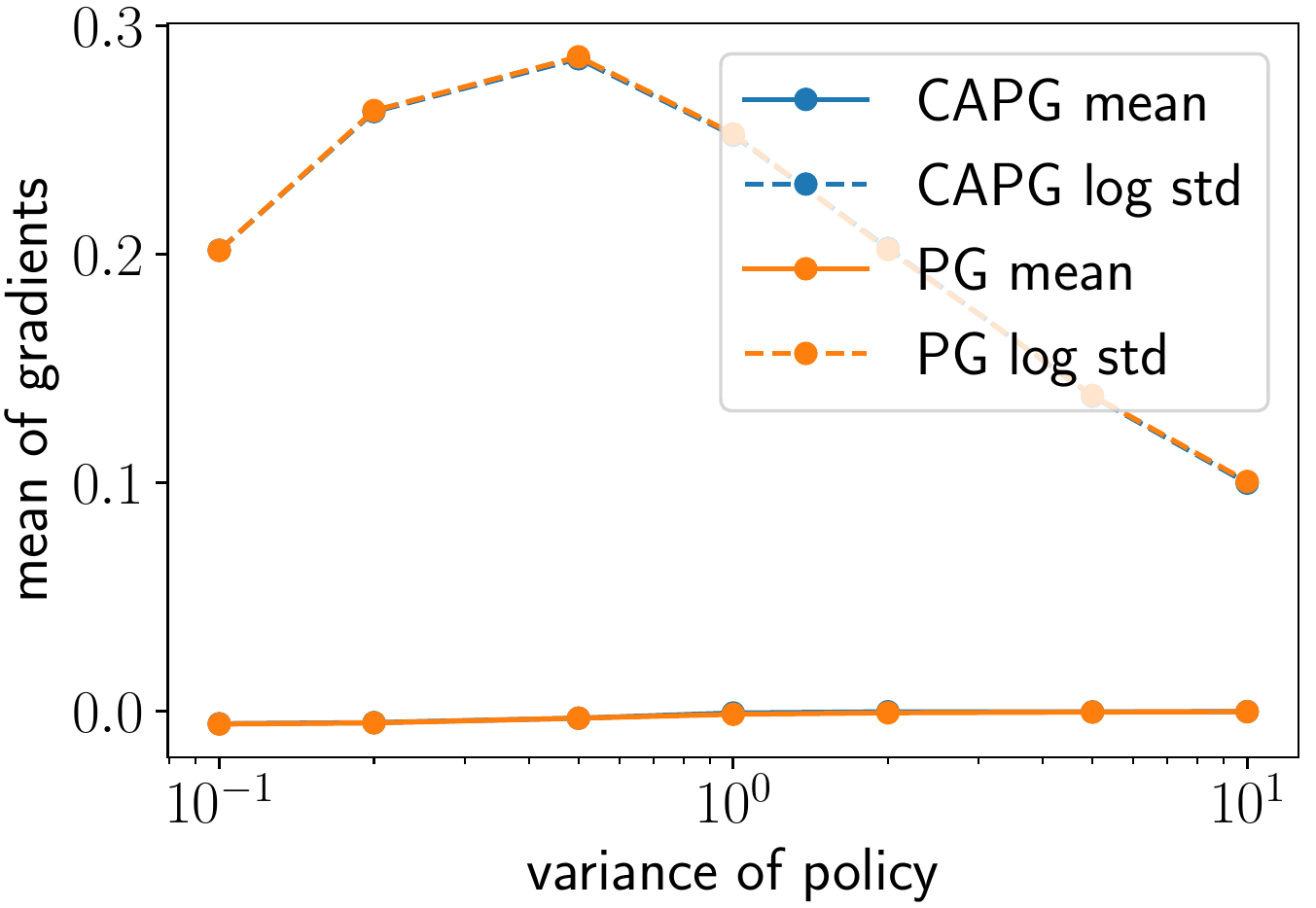}
  \includegraphics[width=0.24\hsize]{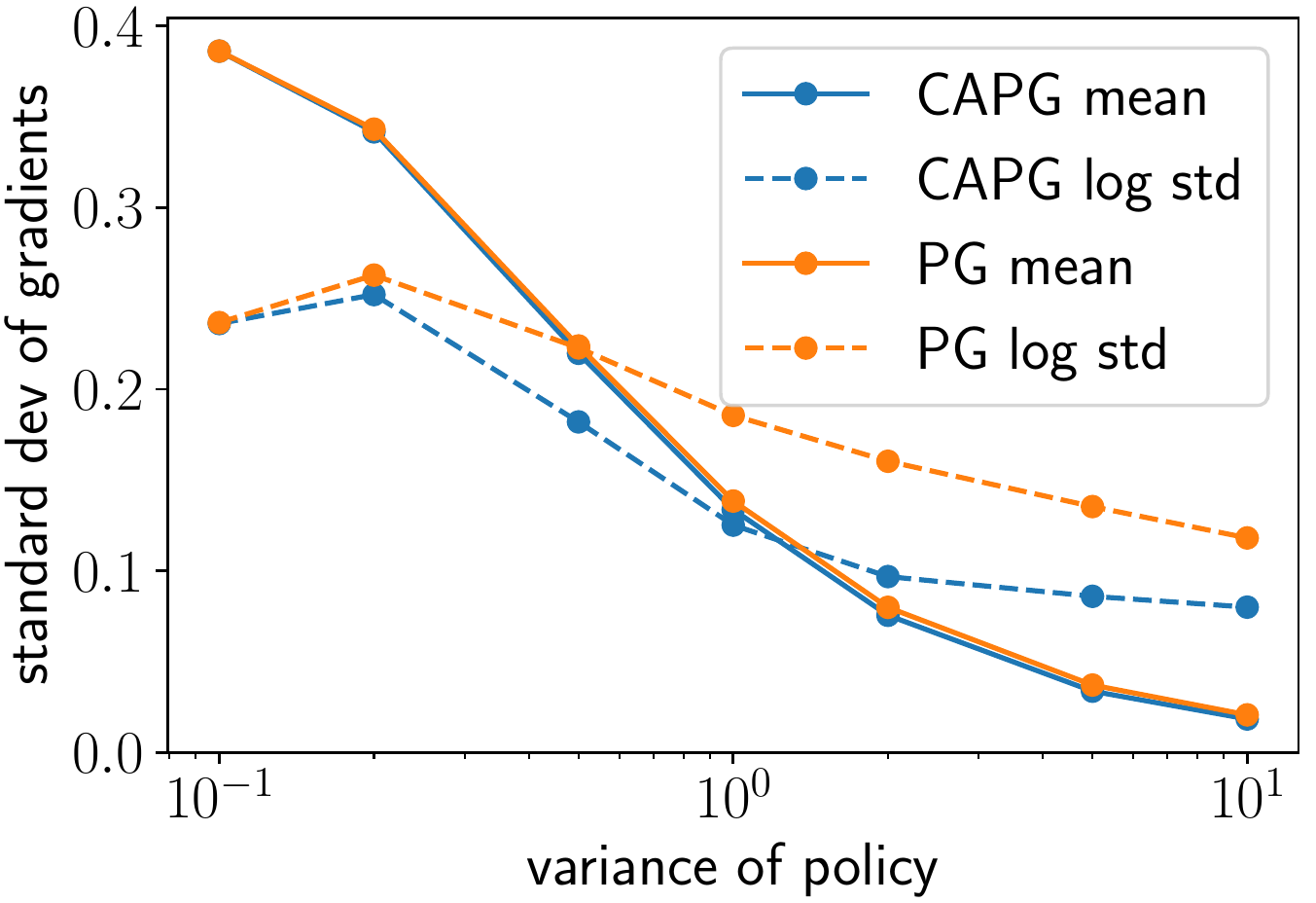}
  \caption{
Means and standard deviations of policy gradient estimates obtained using CAPG and PG on a continuum-armed bandit problem with a fixed policy of varying means (left half) and variances (right half).
For each data point, policy gradients with respect to $\theta_\mu$ and $\theta_\Sigma$ are estimated 10,000 times using 10,000 different batches of 5 (action, reward) pairs.
The CAPG and PG plots of the means of gradients almost overlap each other, and hence, only the PG plots are visible.
}
  \label{fig:stdev_grad}
\end{figure*}
\begin{figure*}[!t]
  \centering
  \includegraphics[width=0.24\hsize]{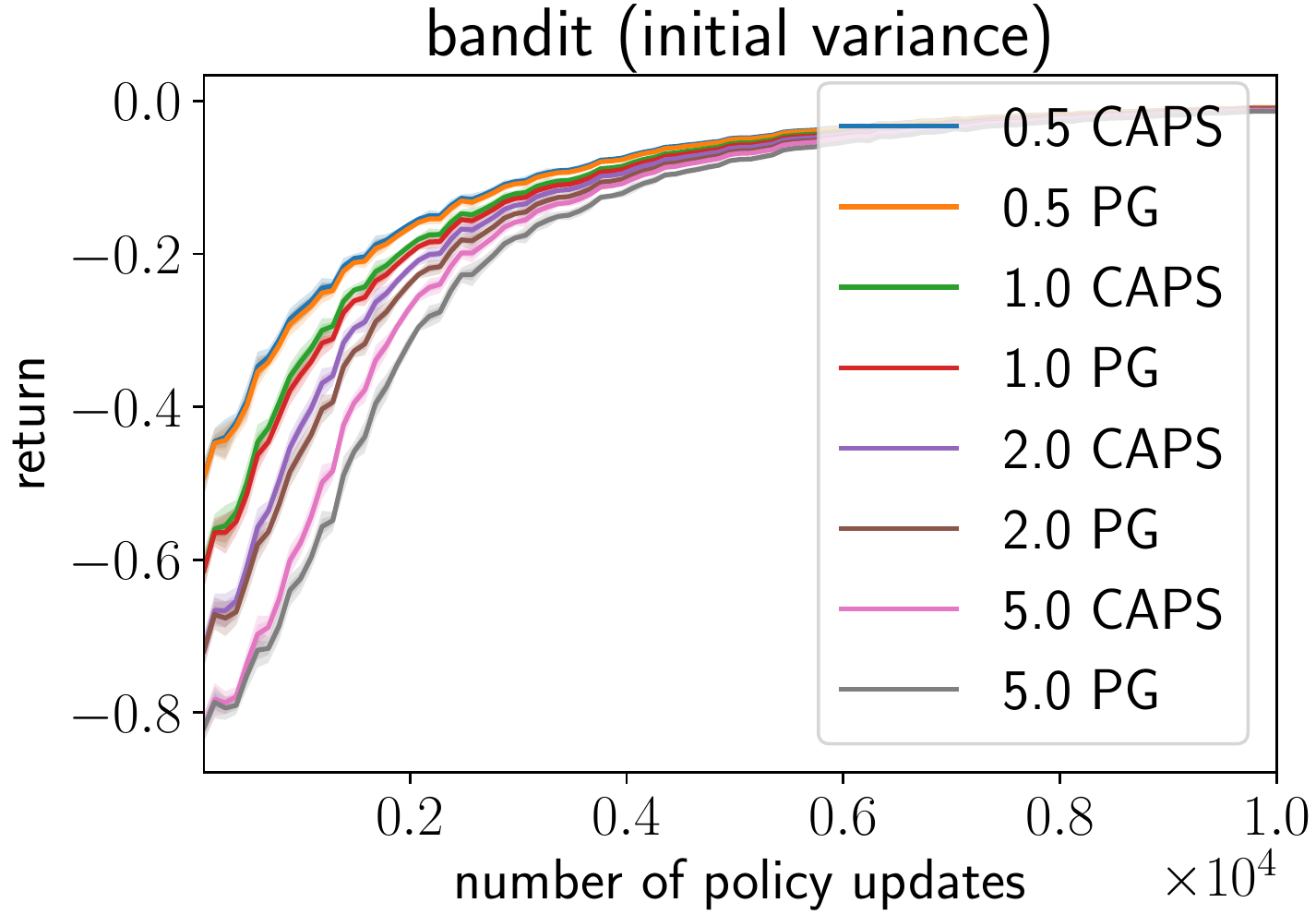}
  \includegraphics[width=0.24\hsize]{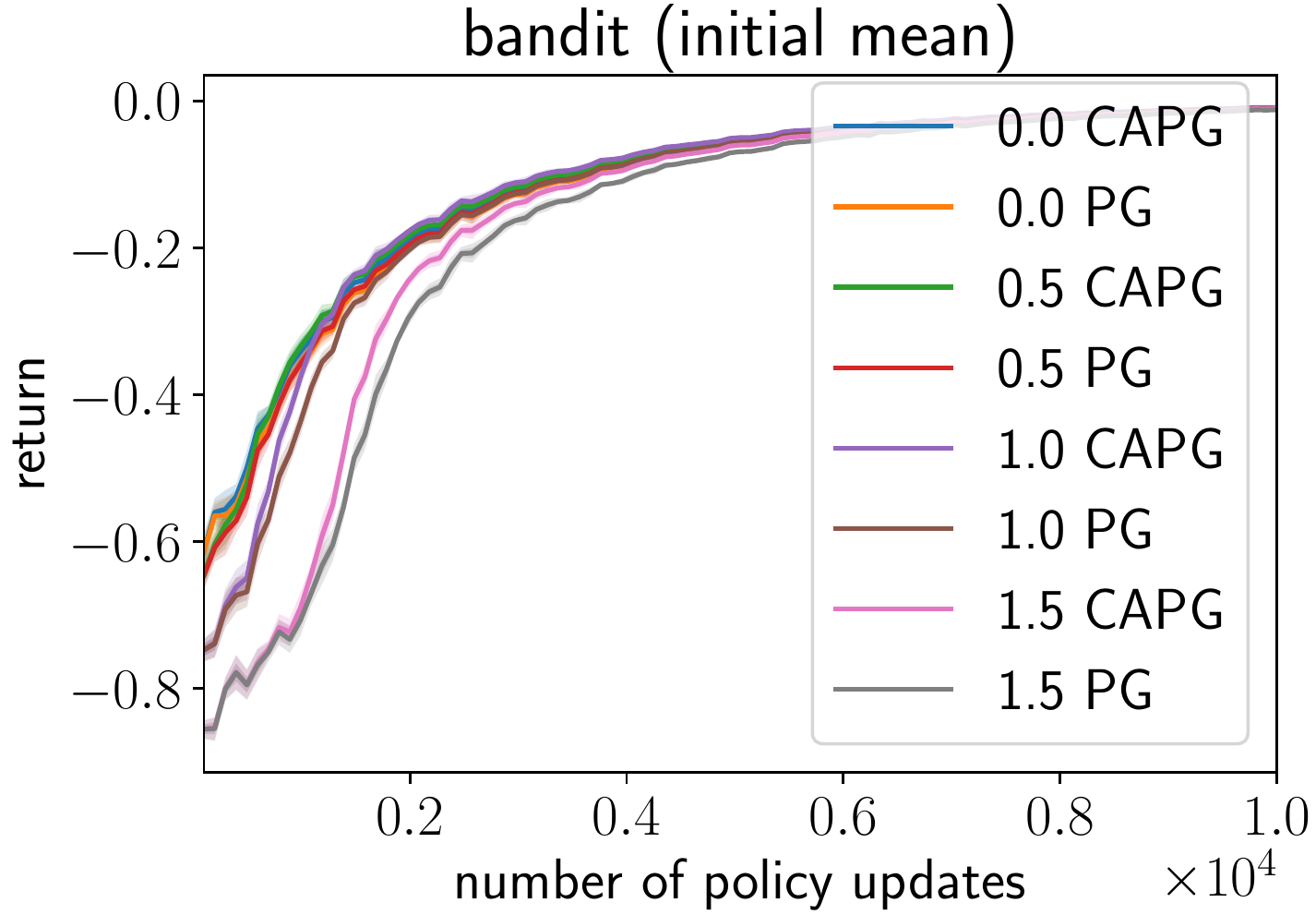}
  \includegraphics[width=0.24\hsize]{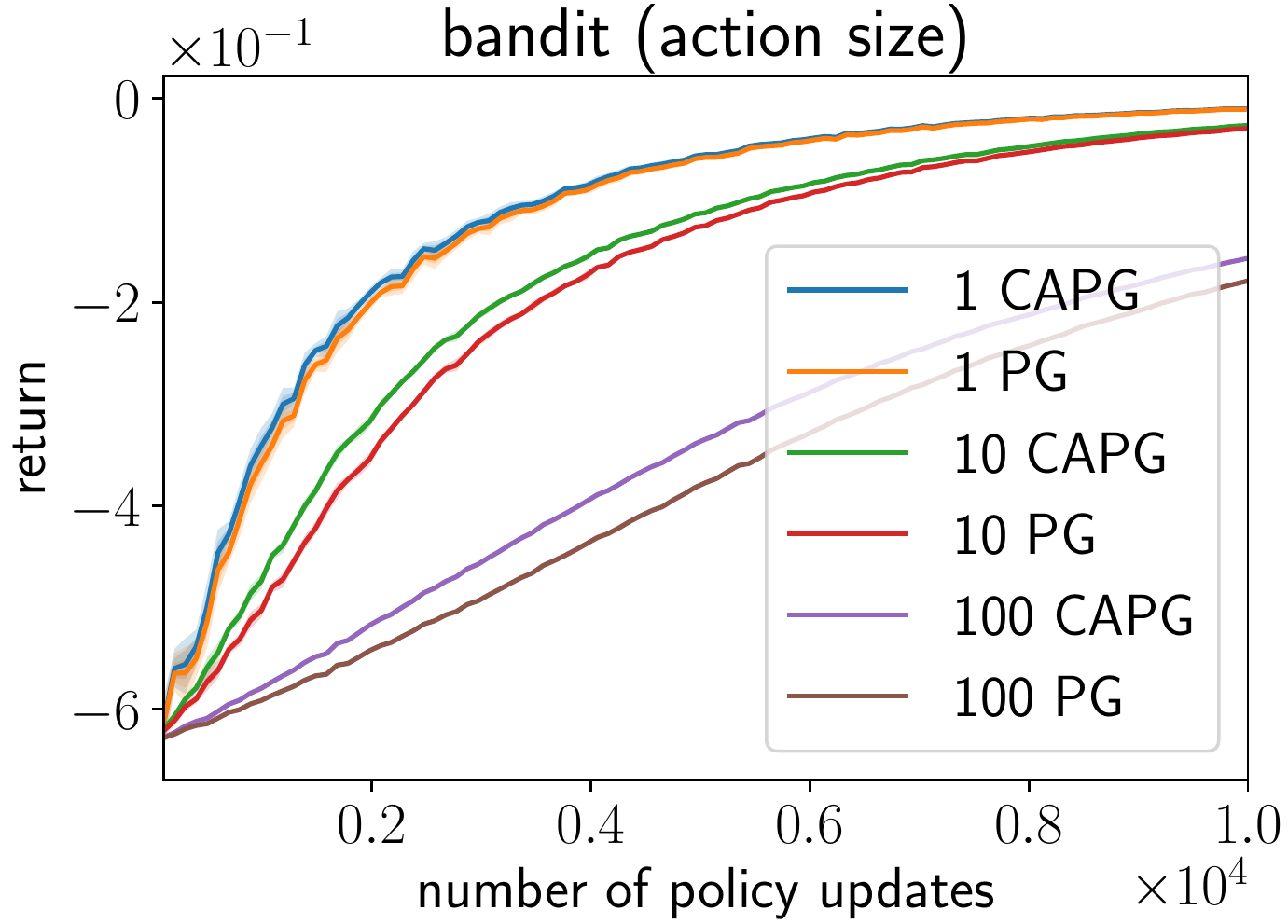}
  \includegraphics[width=0.24\hsize]{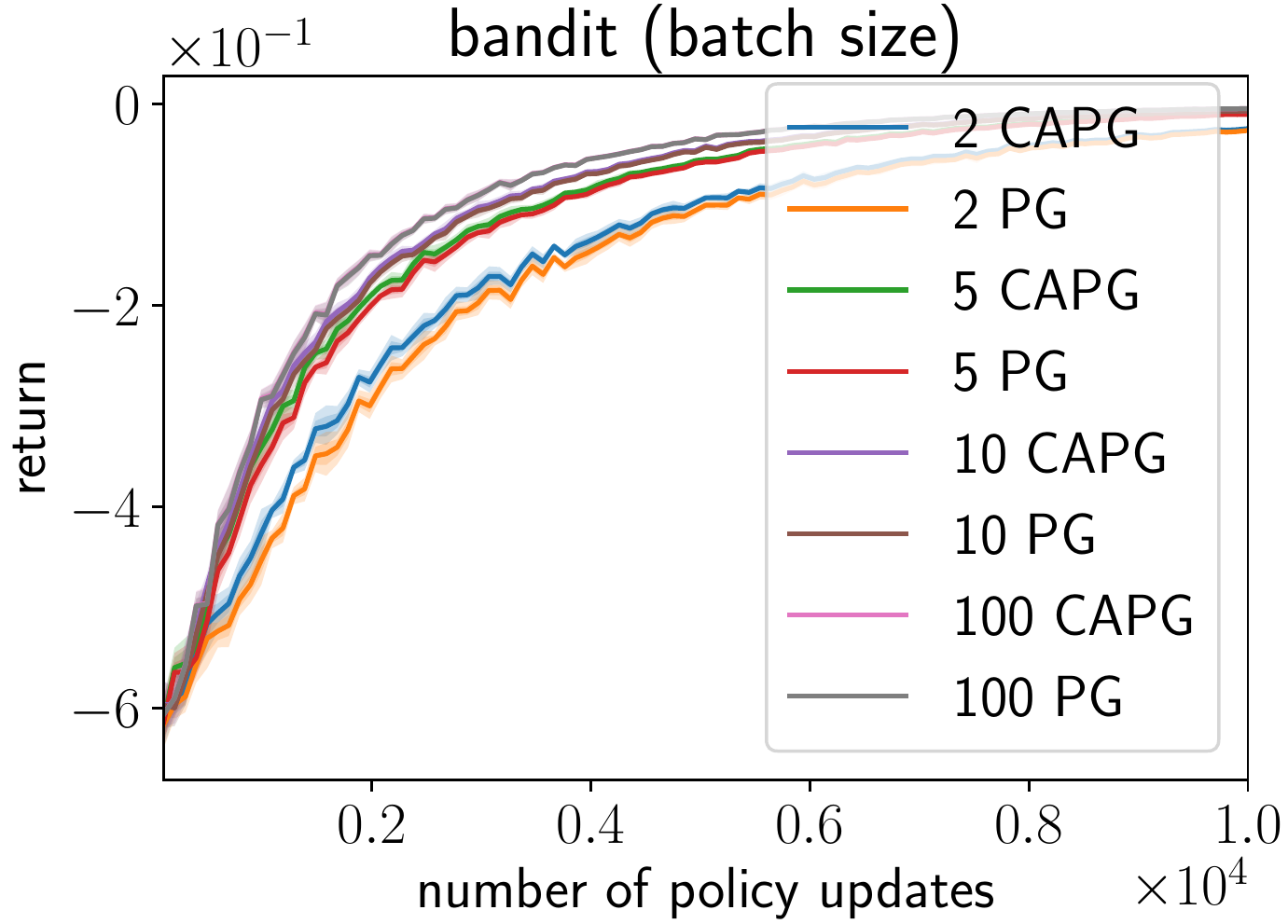}
  \caption{Training curves on continuum-armed bandit problems with four different aspects separately controlled: (from left to right) variance of the initial policy, mean of the initial policy, number of dimensions of actions, and batch size.
For each run, the last reward before every policy update is sampled and then averaged over the previous 100 updates to obtain a smoothed curve.
The smoothed curves are then averaged to compute the mean curves with 68\% and 95\% bootstrapped confidence intervals, which are indicated by the shaded areas.
}
  \label{fig:bandit_curves}
\end{figure*}

To demonstrate how CAPG works and how it interacts with each aspect of problems separately, we used continuum-armed bandit problems \cite{Agrawal1995}, i.e., MDPs with continuous action spaces and no state transitions.
State-independent policies were optimized by policy gradients to maximize action-dependent immediate rewards.

The action space was $[-1,1]^d, d \ge 1$ and the reward function was defined as
$r(u) = -\frac{1}{d} \sum_i |u_i|$
so that only choosing the optimal action of zeros achieves the maximum, zero reward.

Each policy was modeled as a multivariate Gaussian distribution with a diagonal covariance matrix and parameterized by $\theta = \{\theta_\mu, \theta_\Sigma\}$, where $\theta_\mu \in \R^d$ is the mean vector and $\theta_\Sigma \in \R^d$ is the main diagonal of the covariance matrix.

The following experimental settings were used unless otherwise stated.
Actions were scalars, i.e., $d=1$.
The parameters of a policy were initialized as zero mean and unit variance for each dimension.
Each policy update used a batch of 5 (action, reward) pairs.
The average reward in a batch was used as a baseline that was subtracted from each reward.
Adam \cite{Kingma2015b} with its default hyperparameters was used to update the parameters.

To quantify the variance reduction achieved by CAPG, we repeatedly estimated policy gradients using new samples without updating a policy.
Figure~\ref{fig:stdev_grad} shows the mean and standard deviation of policy gradient estimates obtained by CAPG and PG with a fixed policy of varying means and variances.
For both $\theta_\mu$ and $\theta_\Sigma$ in all settings, CAPG consistently achieved lower variance than PG without introducing visible bias.
These results numerically corroborate CAPG's variance reduction ability as well as its unbiasedness.
The efficacy of CAPG diminished at $\sigma^2=0.1$, where sampled actions rarely go outside the bounds.

Figure~\ref{fig:bandit_curves} shows the training curves of CAPG and PG with four different aspects separately controlled: variance of the initial policy, mean of the initial policy, number of dimensions of actions, and batch size.
Each configuration is evaluated with 10 different random seeds.
CAPG consistently achieved faster learning across the settings.
A larger initial variance and a more distant initial mean tend to make the gap more visible.
CAPG's gain scales even for 100 dimensions, implying its utility for more challenging, complex continuous control tasks.
Using smaller batch sizes benefits more from CAPG, and this is expected because smaller batch sizes are more affected by the variance of gradient estimation.
With the batch size of 100, the training curve of CAPG is difficult to distinguish from that of PG.
It should be noted that in these experiments all the actions are sampled from the same state.
In practical model-free RL scenarios, more than one action cannot be sampled from the same state.

\subsection{Simulated control problems}

\begin{table}
  \small
  \centering
  \begin{tabular}{lll}
  \toprule
  {} & Obs. space &        Action space \\
  \midrule
  InvertedPendulum-v1       &         $\R^{4}$ &   $[-3.0, 3.0]^{1}$ \\
  InvertedDoublePendulum-v1 &        $\R^{11}$ &   $[-1.0, 1.0]^{1}$ \\
  Reacher-v1                &        $\R^{11}$ &   $[-1.0, 1.0]^{2}$ \\
  Hopper-v1                 &        $\R^{11}$ &   $[-1.0, 1.0]^{3}$ \\
  HalfCheetah-v1            &        $\R^{17}$ &   $[-1.0, 1.0]^{6}$ \\
  Swimmer-v1                &         $\R^{8}$ &   $[-1.0, 1.0]^{2}$ \\
  Walker2d-v1               &        $\R^{17}$ &   $[-1.0, 1.0]^{6}$ \\
  Ant-v1                    &       $\R^{111}$ &   $[-1.0, 1.0]^{8}$ \\
  Humanoid-v1               &       $\R^{376}$ &  $[-0.4, 0.4]^{17}$ \\
  HumanoidStandup-v1        &       $\R^{376}$ &  $[-0.4, 0.4]^{17}$ \\
  \bottomrule
  \end{tabular}
  \caption{MuJoCo-simulated environments used in the experiments and their observation and action spaces.}
  \label{tab:mujoco_envs}
\end{table}

To evaluate CAPG's effectiveness in more practical settings, we used the following two popular deep RL algorithms for continuous control:
\begin{itemize}
  \item Proximal policy optimization (PPO) with clipped surrogate objective \cite{Schulman2017b}
  \item Trust region policy optimization (TRPO) \cite{Schulman2015e} with generalized advantage estimation (GAE) \cite{Schulman2016}.
\end{itemize}
For each of the two algorithms, we implemented the variant that uses CAPG as well as the original one that uses PG.
The only difference between these two is whether CAPG or PG is used.

For our experiments, we used 10 MuJoCo-simulated environments implemented in OpenAI Gym that are widely used as benchmark tasks for deep RL algorithms \cite{Schulman2017b, Henderson2017a, Ciosek2017a, Gu2017b, Duan2016, baselines}.
The names of the environments are listed along with their observation and action spaces in Table~\ref{tab:mujoco_envs}.
All the environments have bounded action spaces; hence, actions are clipped before being sent to the environments.

\begin{figure*}[!t]
  \centering
  \includegraphics[width=0.196\textwidth]{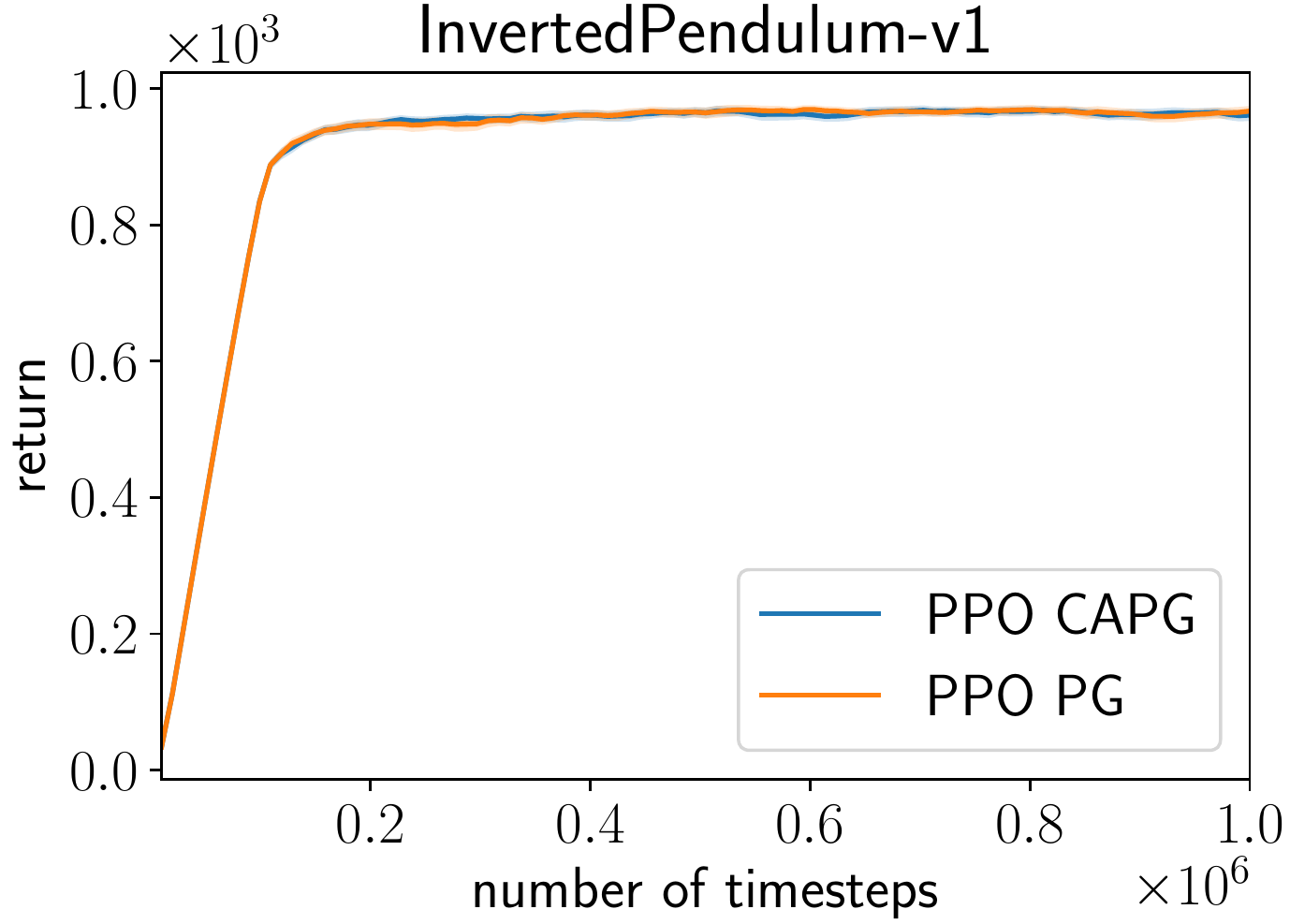}
  \includegraphics[width=0.196\textwidth]{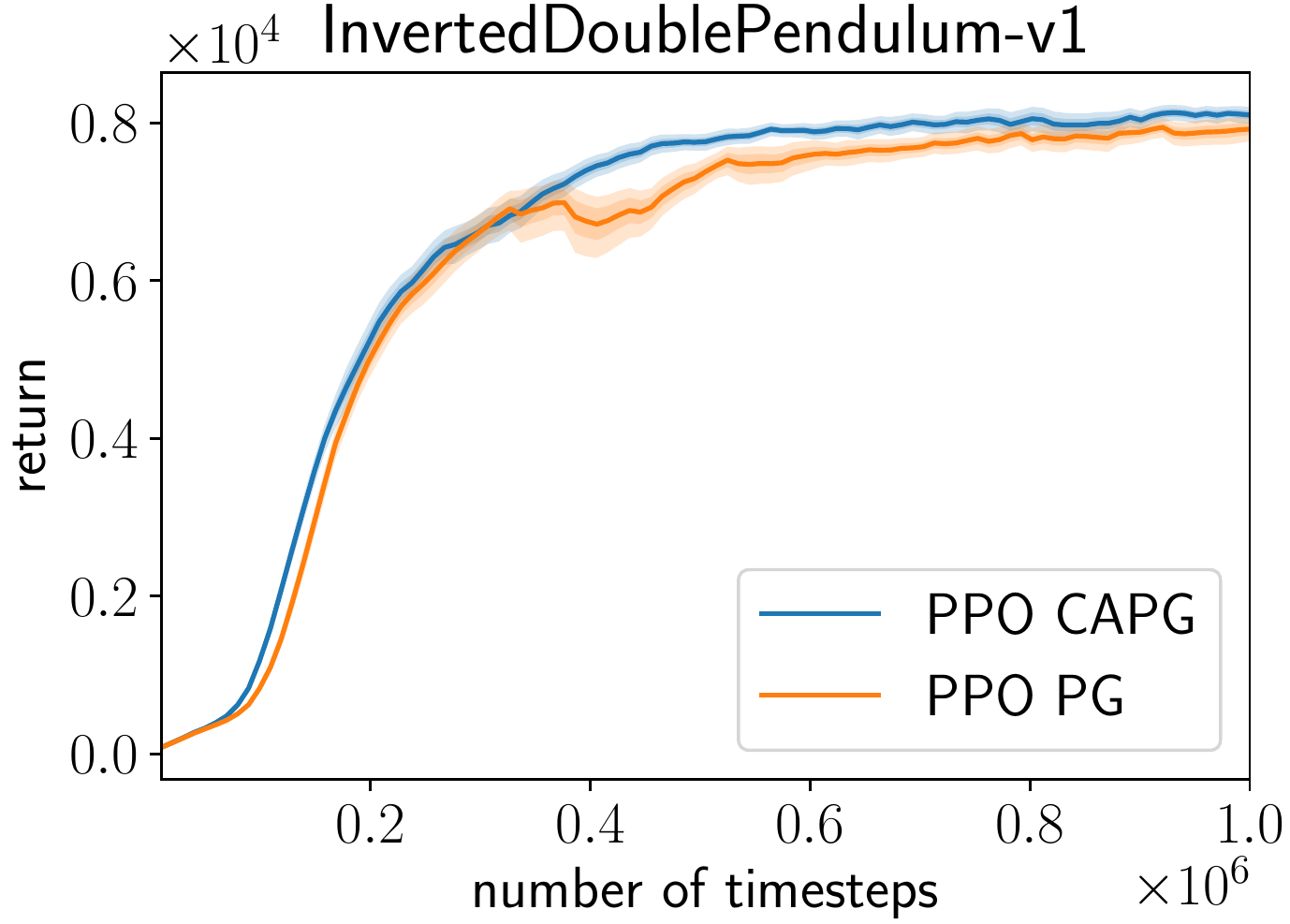}
  \includegraphics[width=0.196\textwidth]{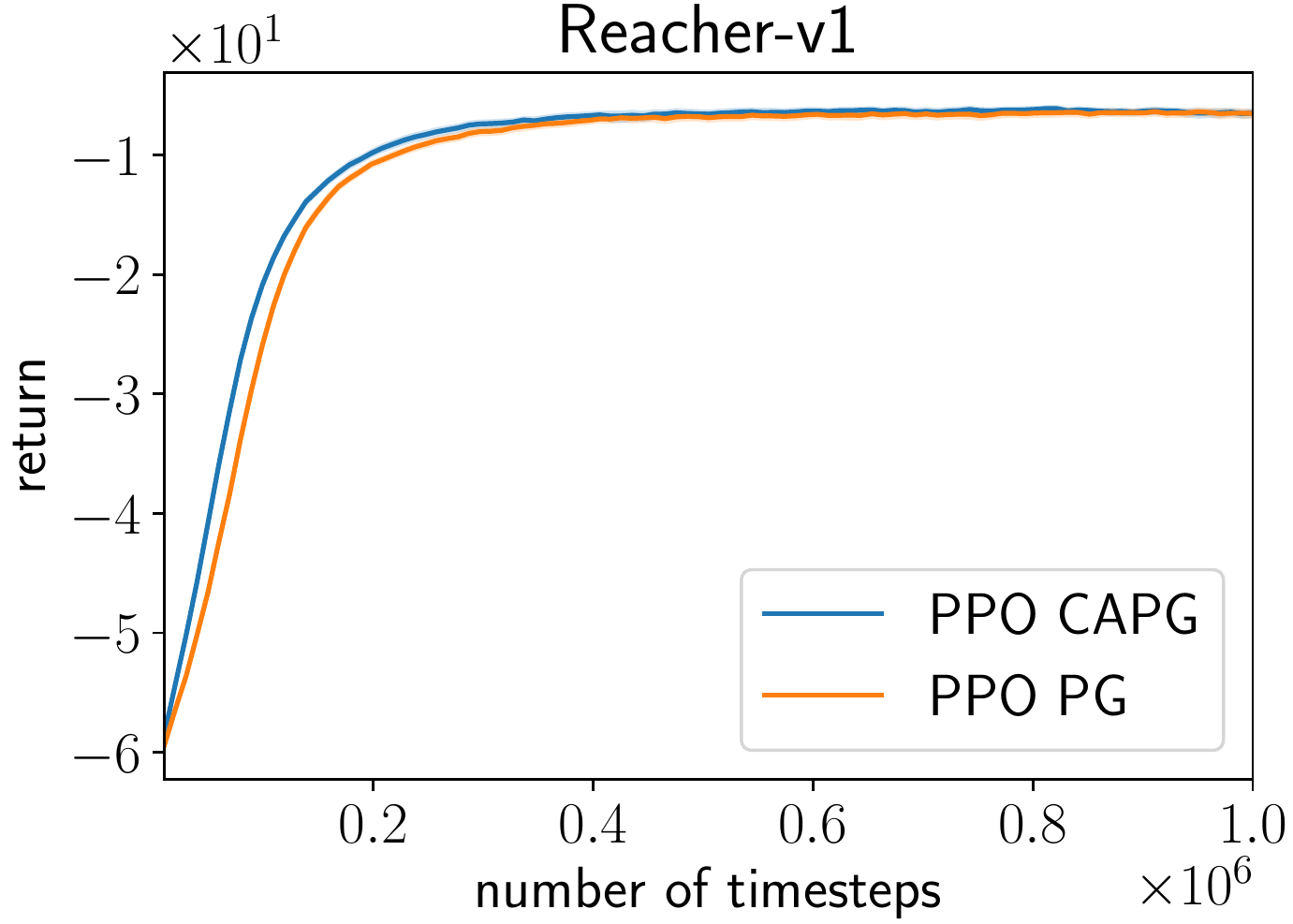}
  \includegraphics[width=0.196\textwidth]{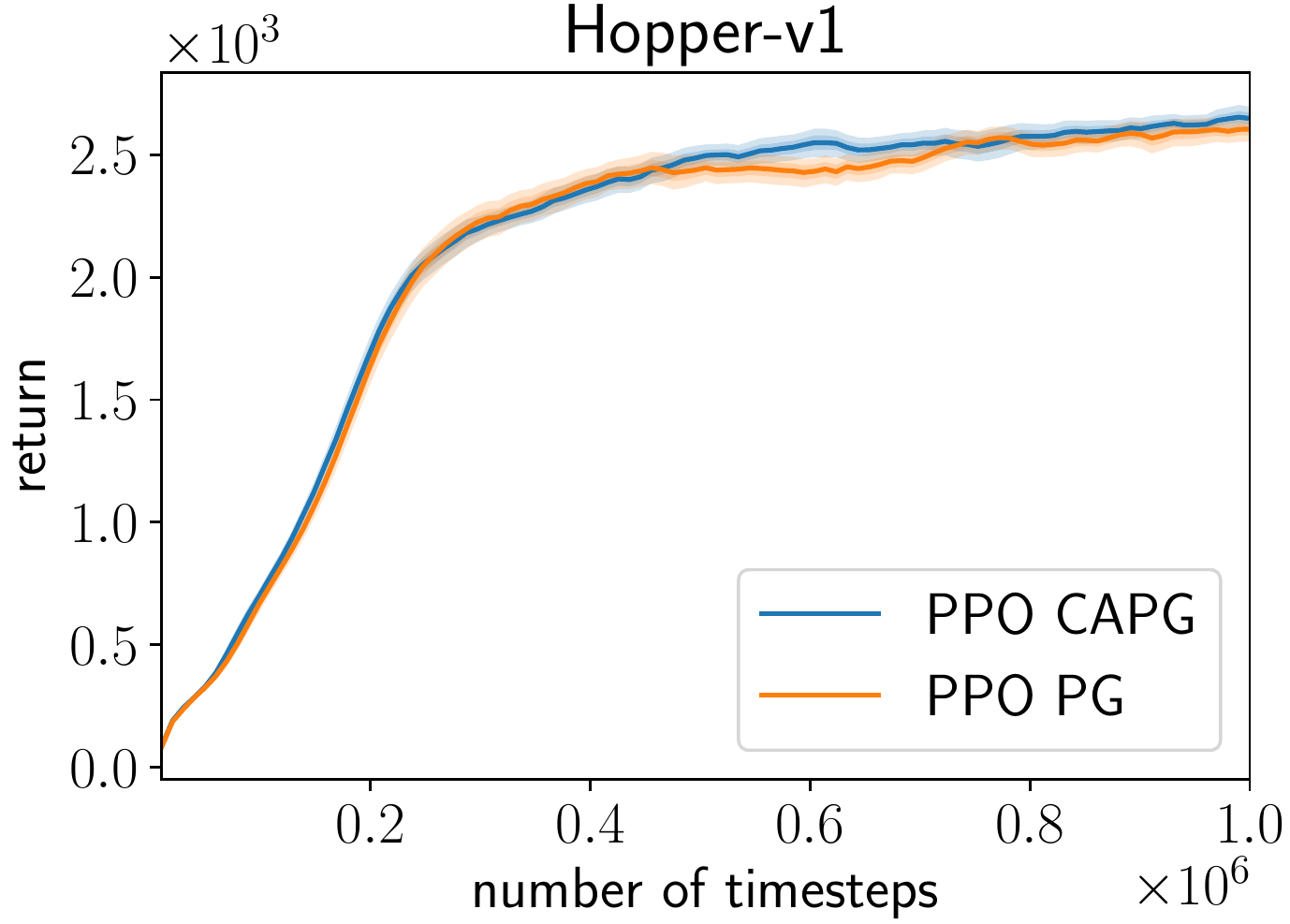}
  \includegraphics[width=0.196\textwidth]{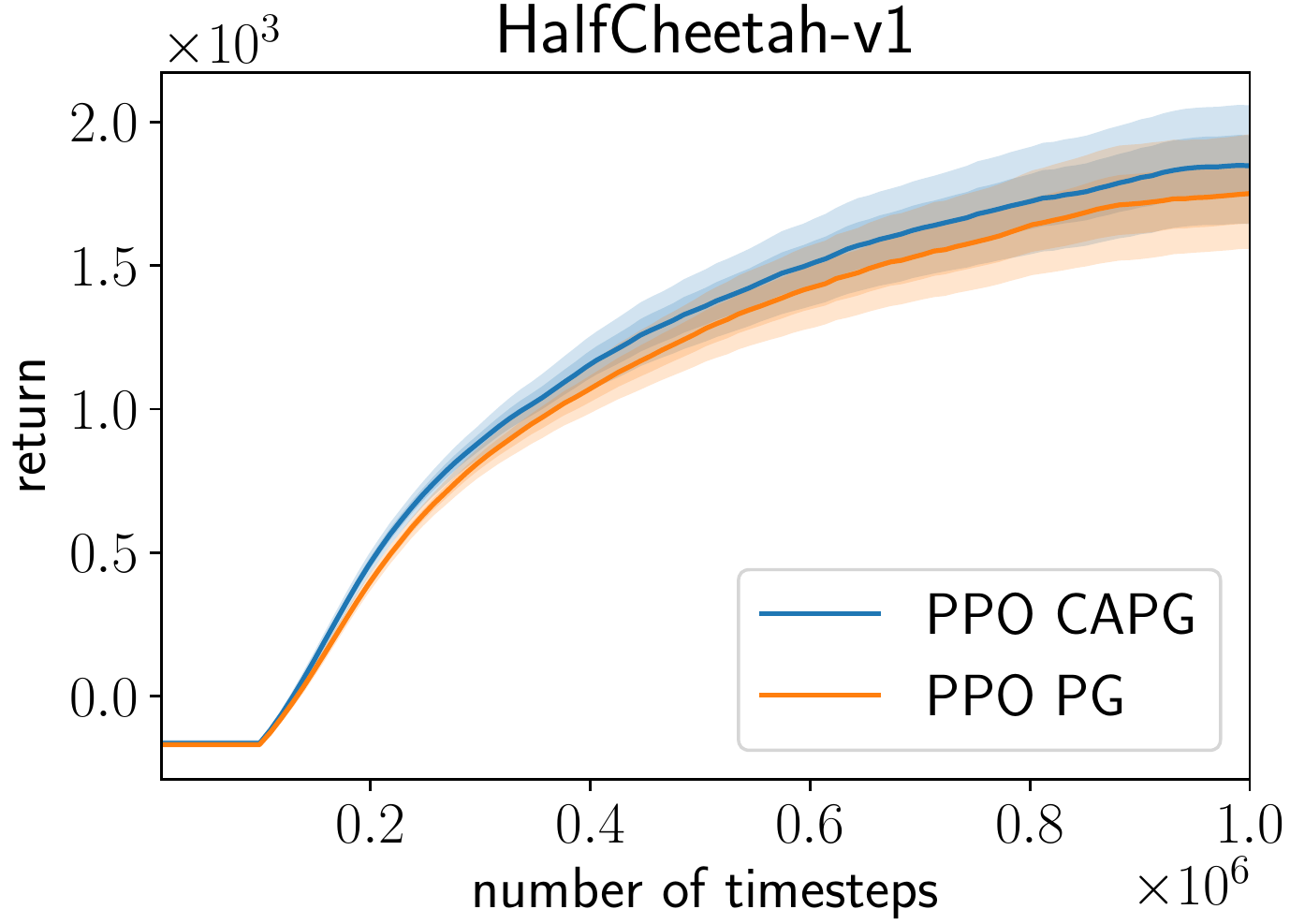}
  \includegraphics[width=0.196\textwidth]{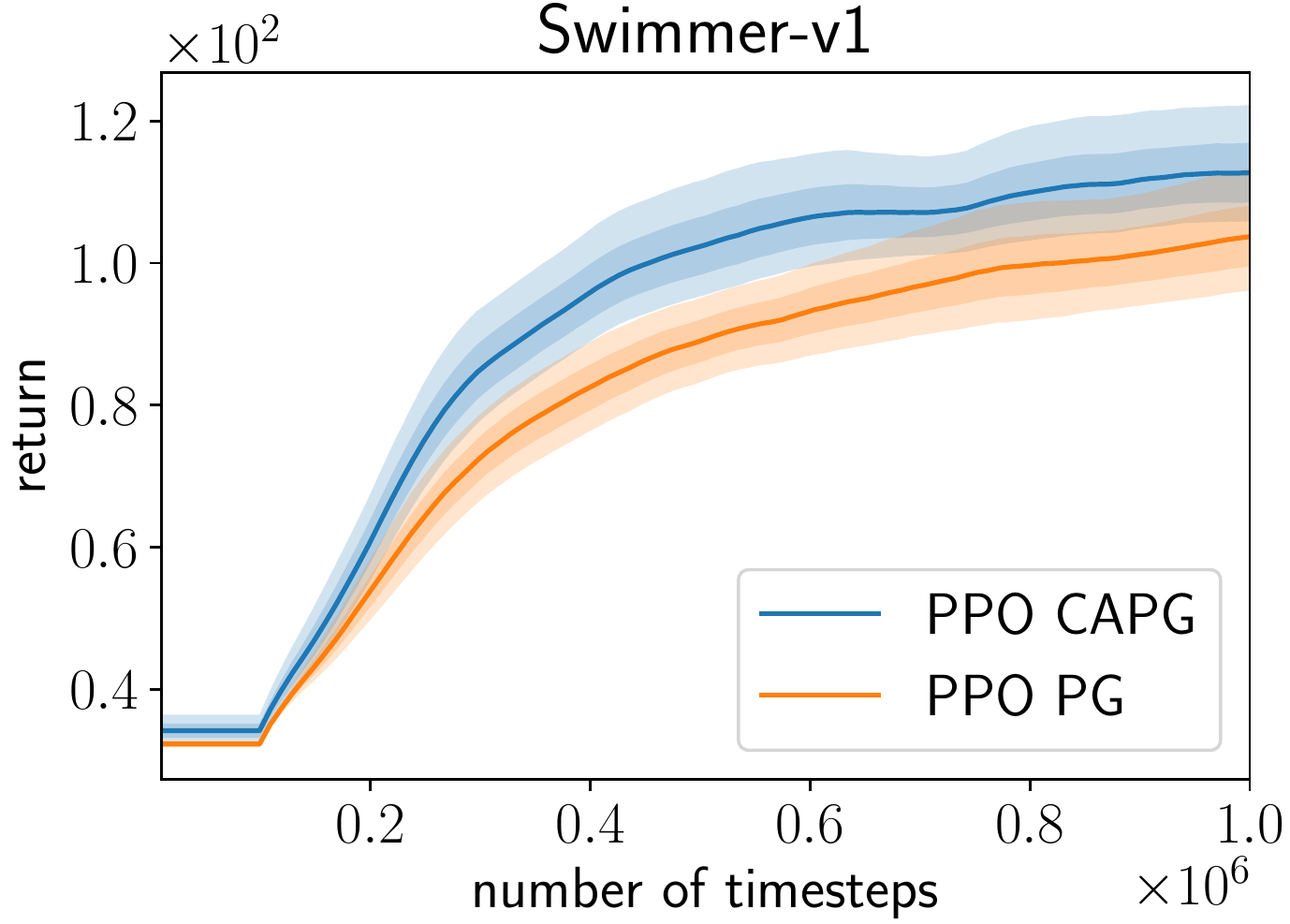}
  \includegraphics[width=0.196\textwidth]{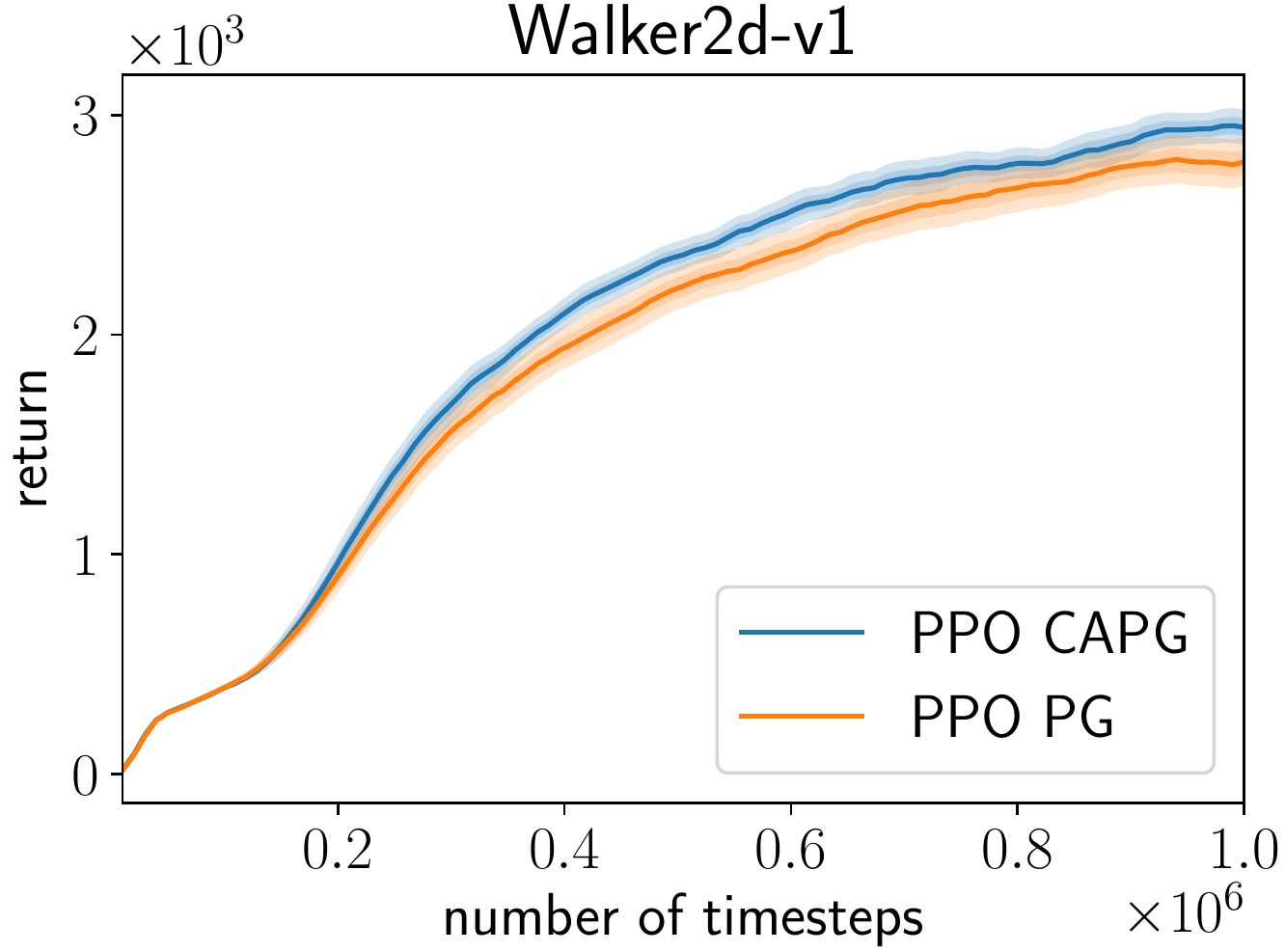}
  \includegraphics[width=0.196\textwidth]{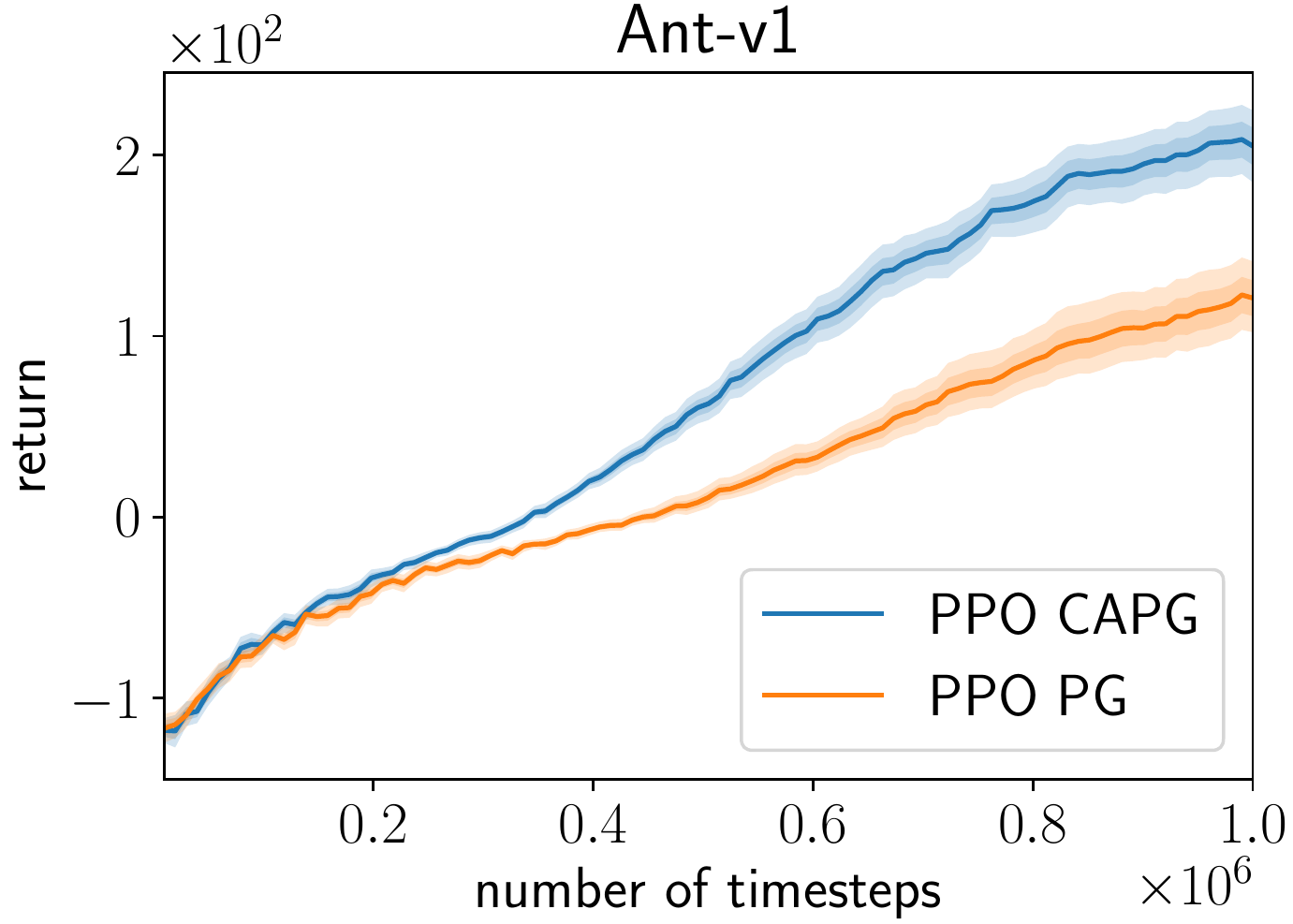}
  \includegraphics[width=0.196\textwidth]{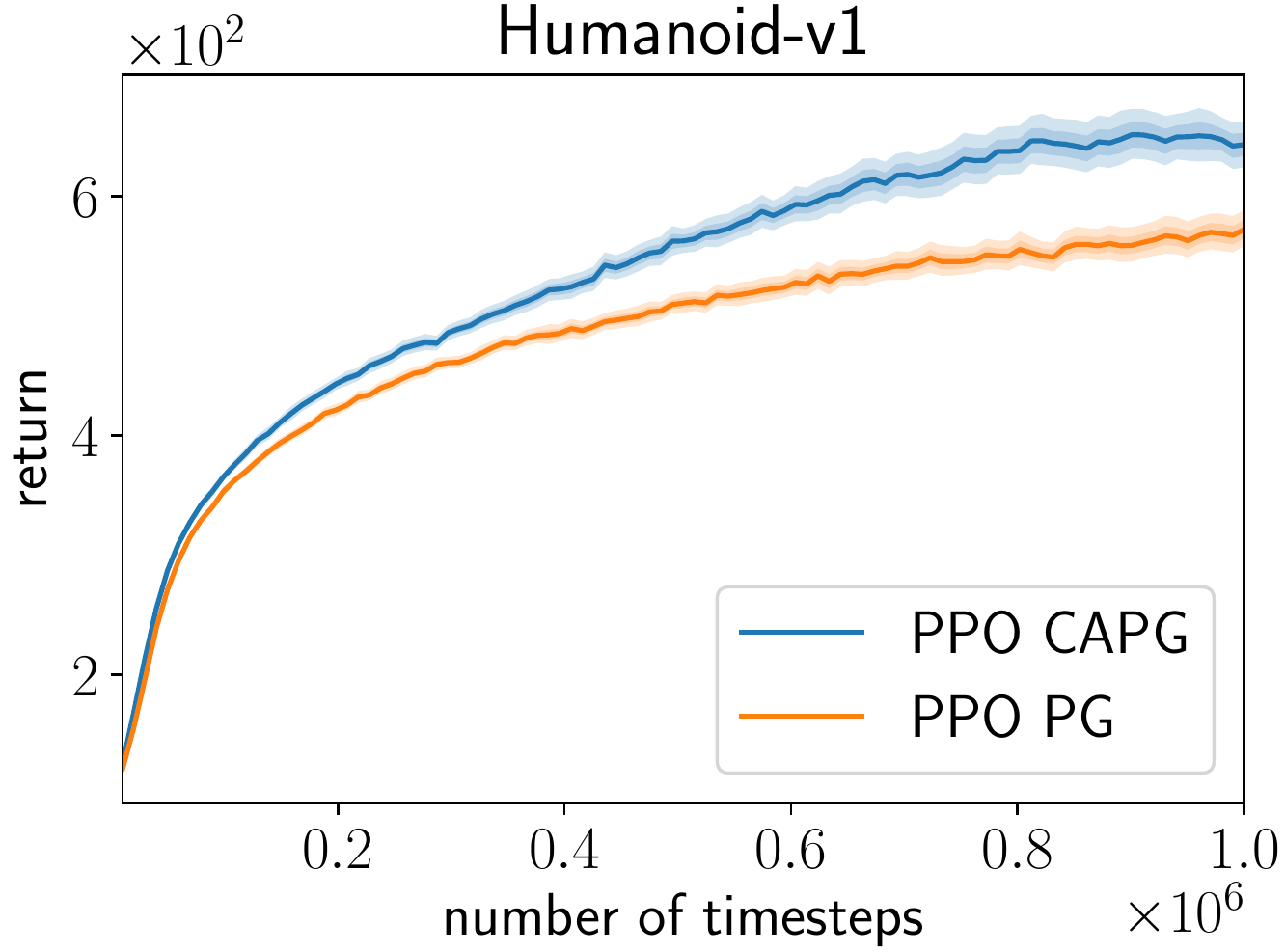}
  \includegraphics[width=0.196\textwidth]{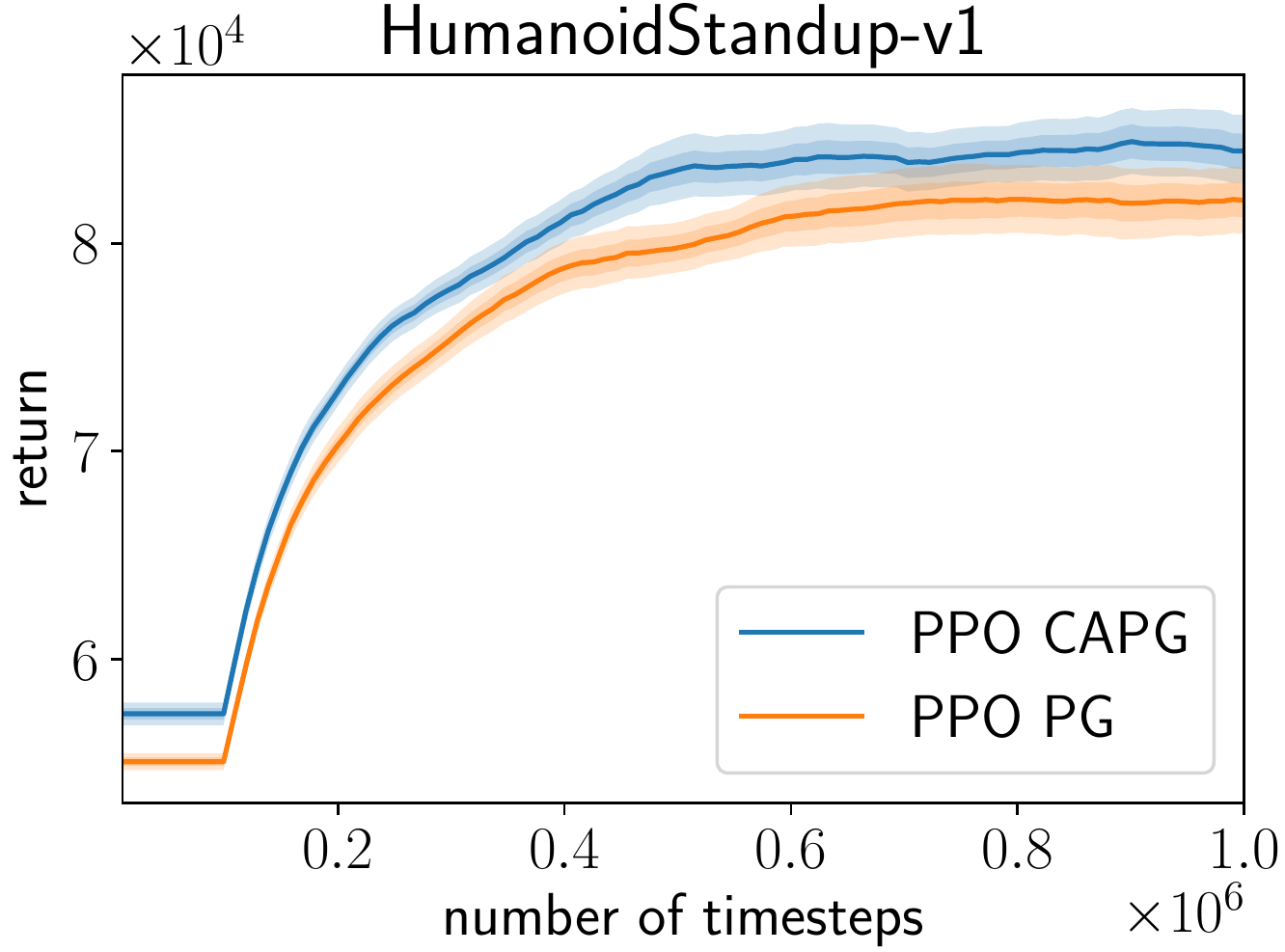}
  \includegraphics[width=0.196\textwidth]{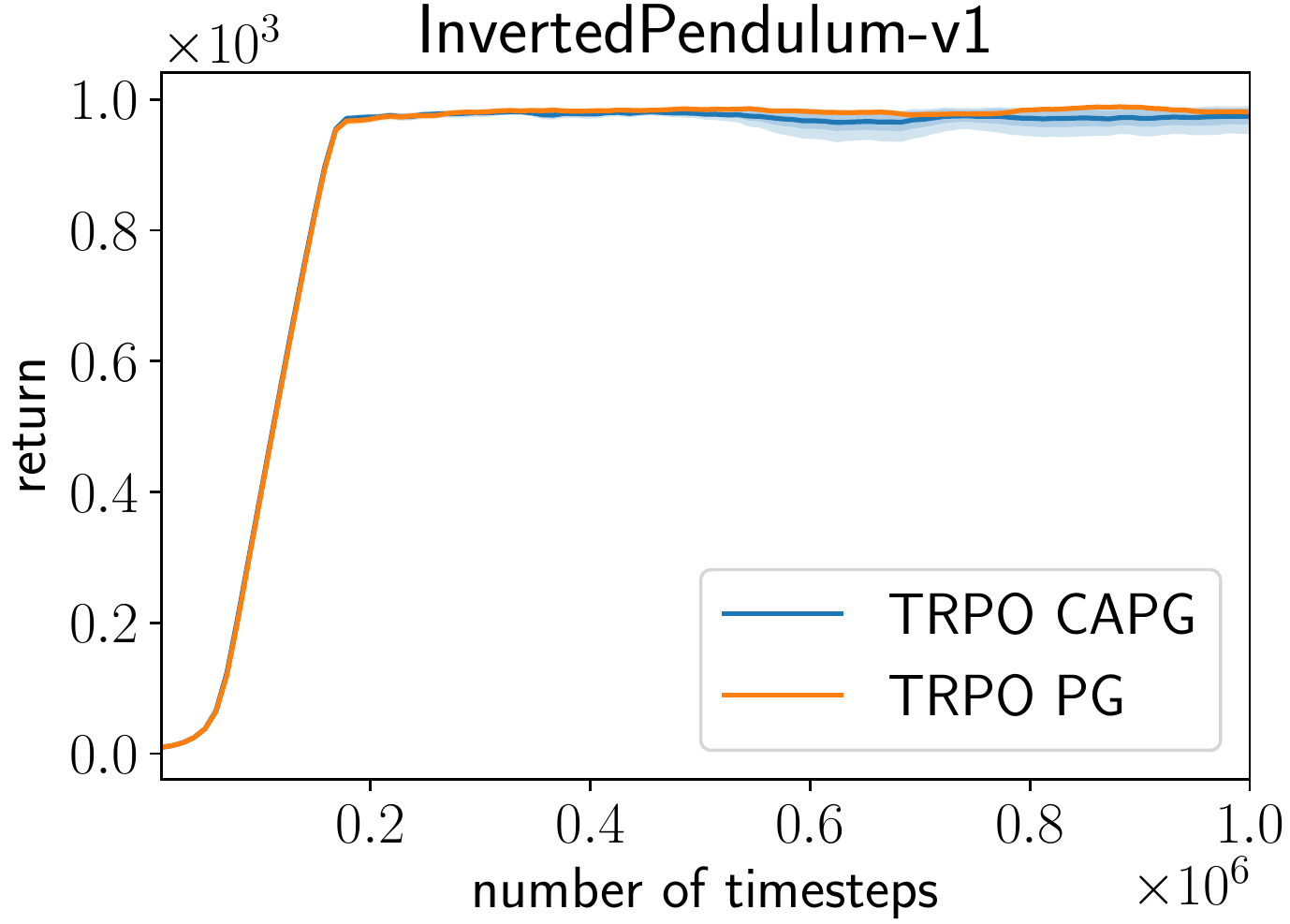}
  \includegraphics[width=0.196\textwidth]{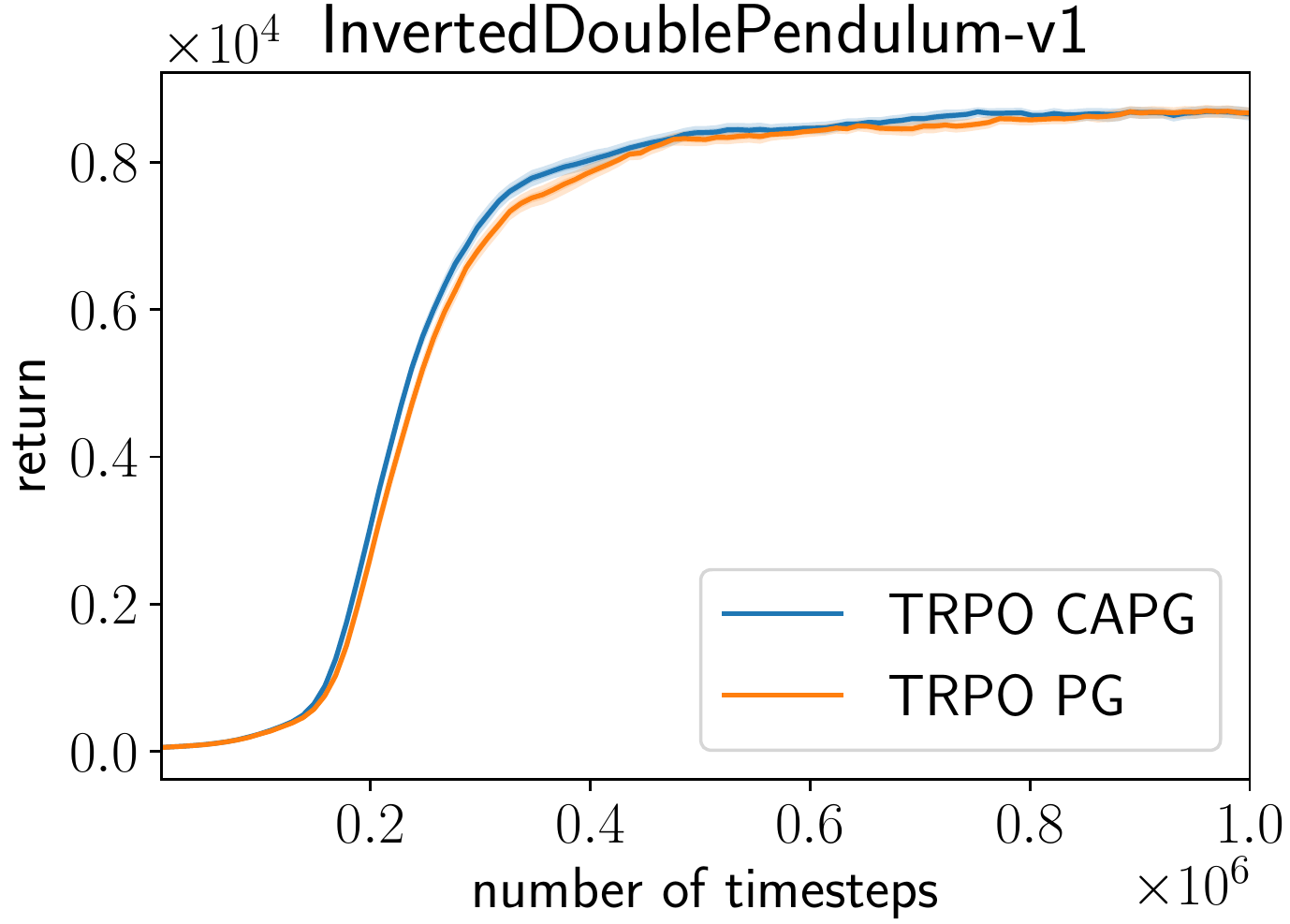}
  \includegraphics[width=0.196\textwidth]{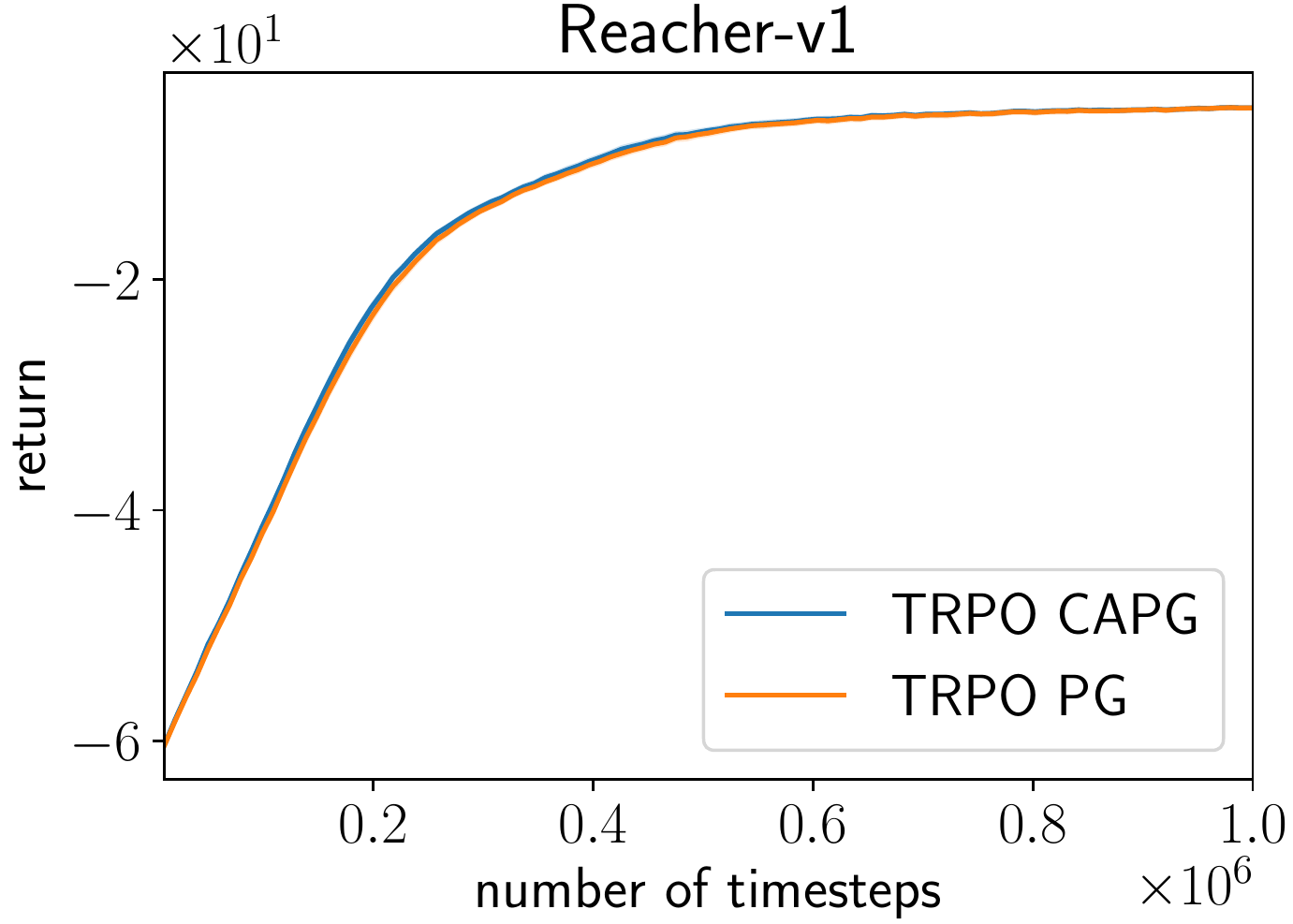}
  \includegraphics[width=0.196\textwidth]{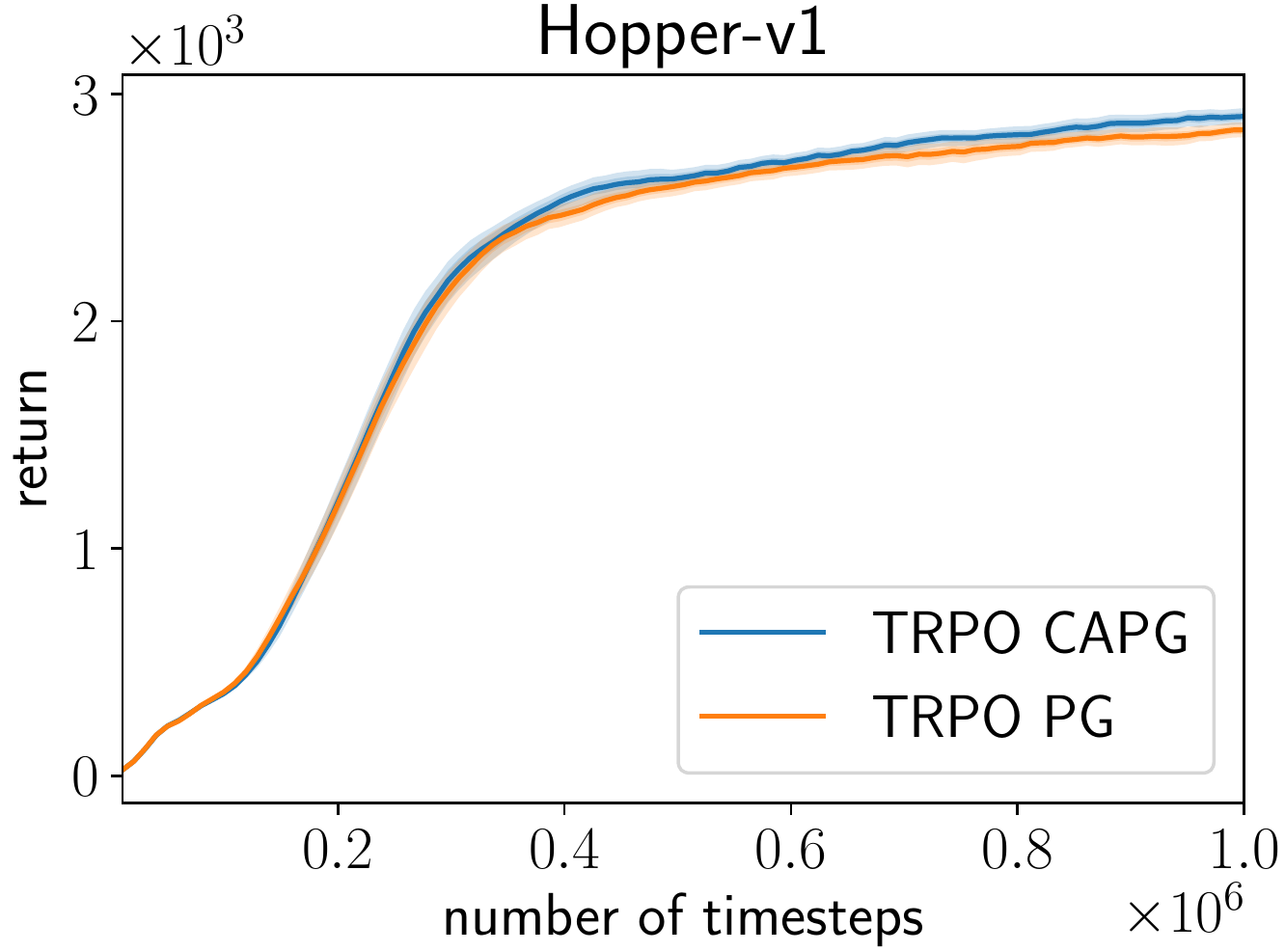}
  \includegraphics[width=0.196\textwidth]{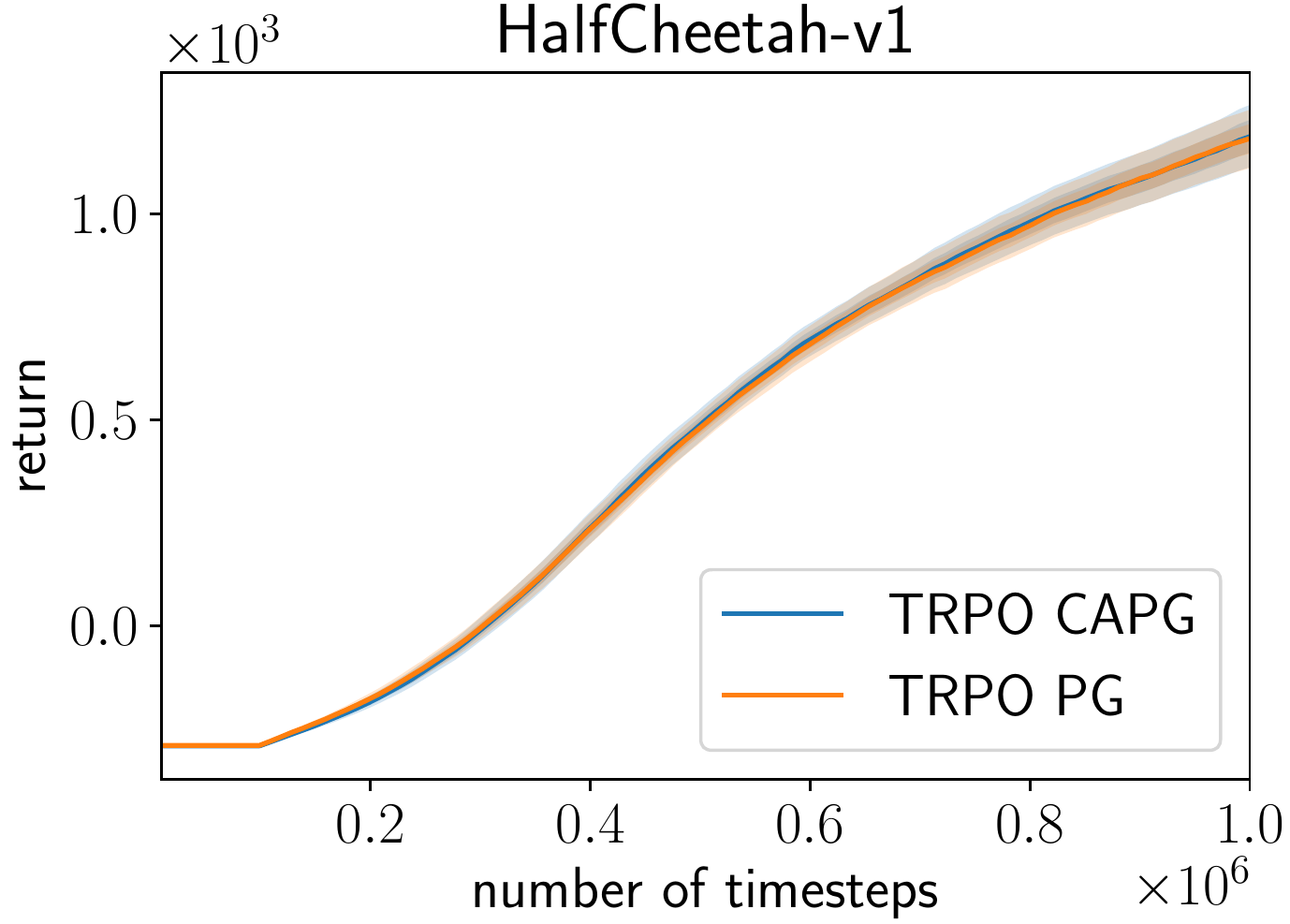}
  \includegraphics[width=0.196\textwidth]{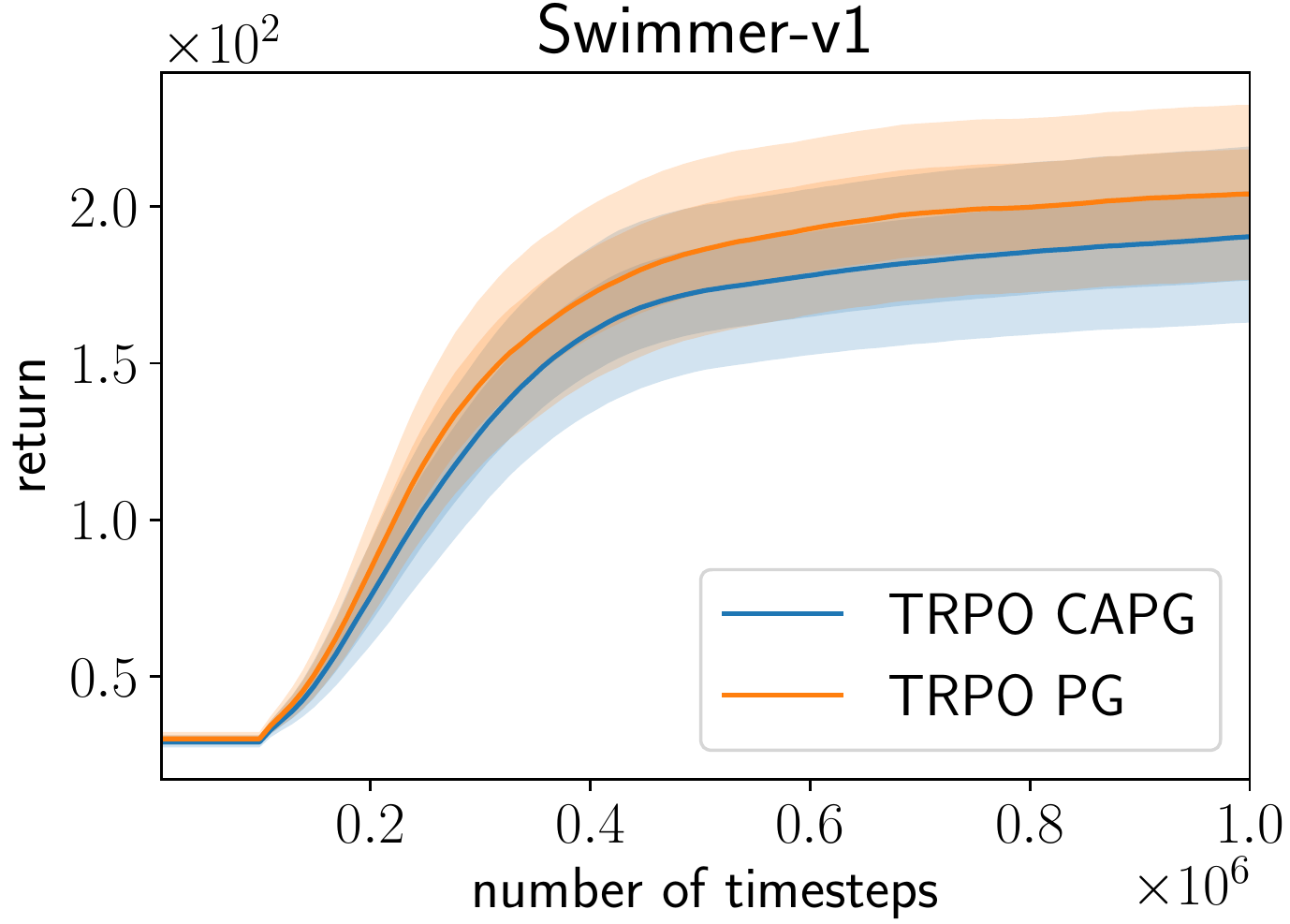}
  \includegraphics[width=0.196\textwidth]{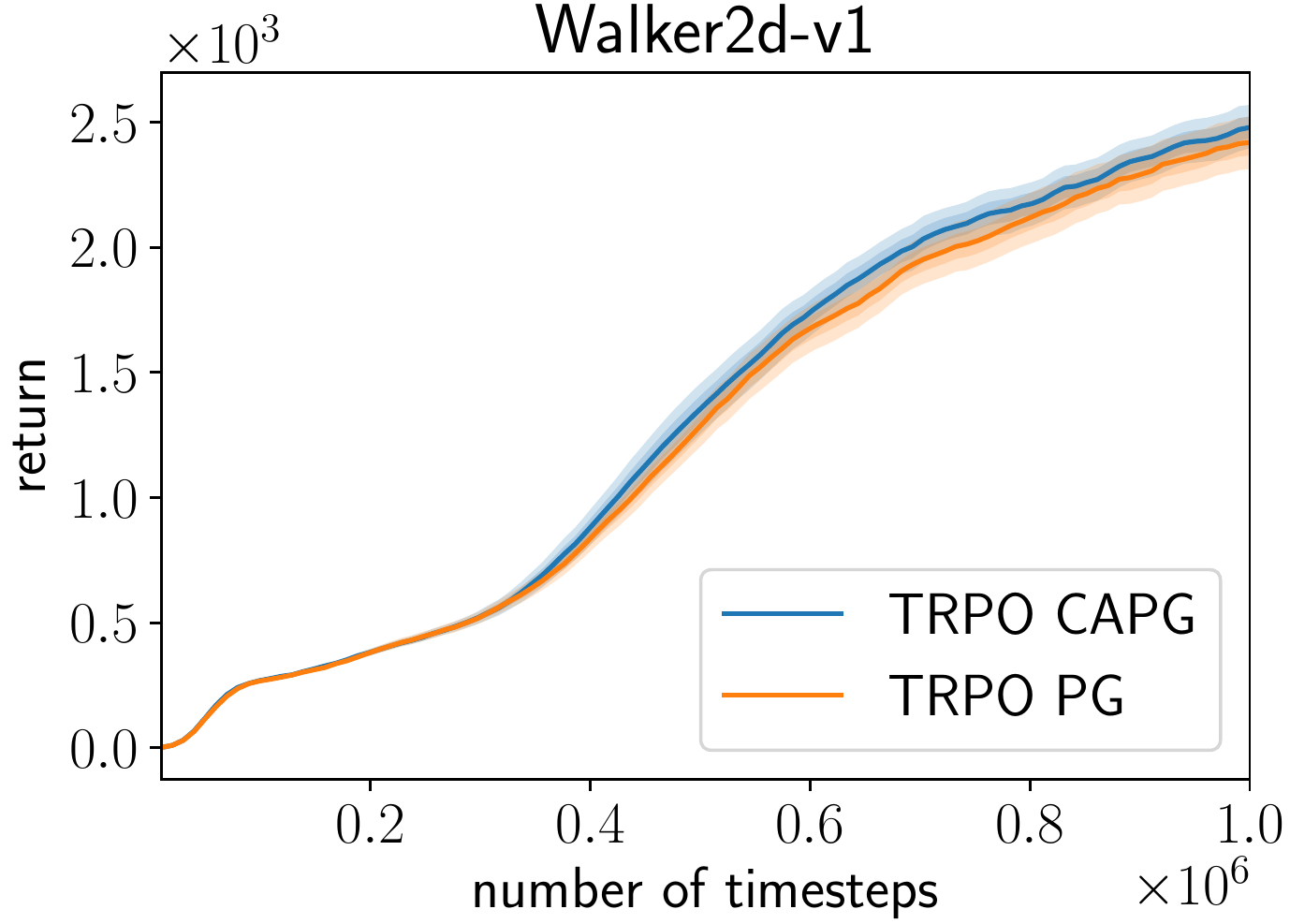}
  \includegraphics[width=0.196\textwidth]{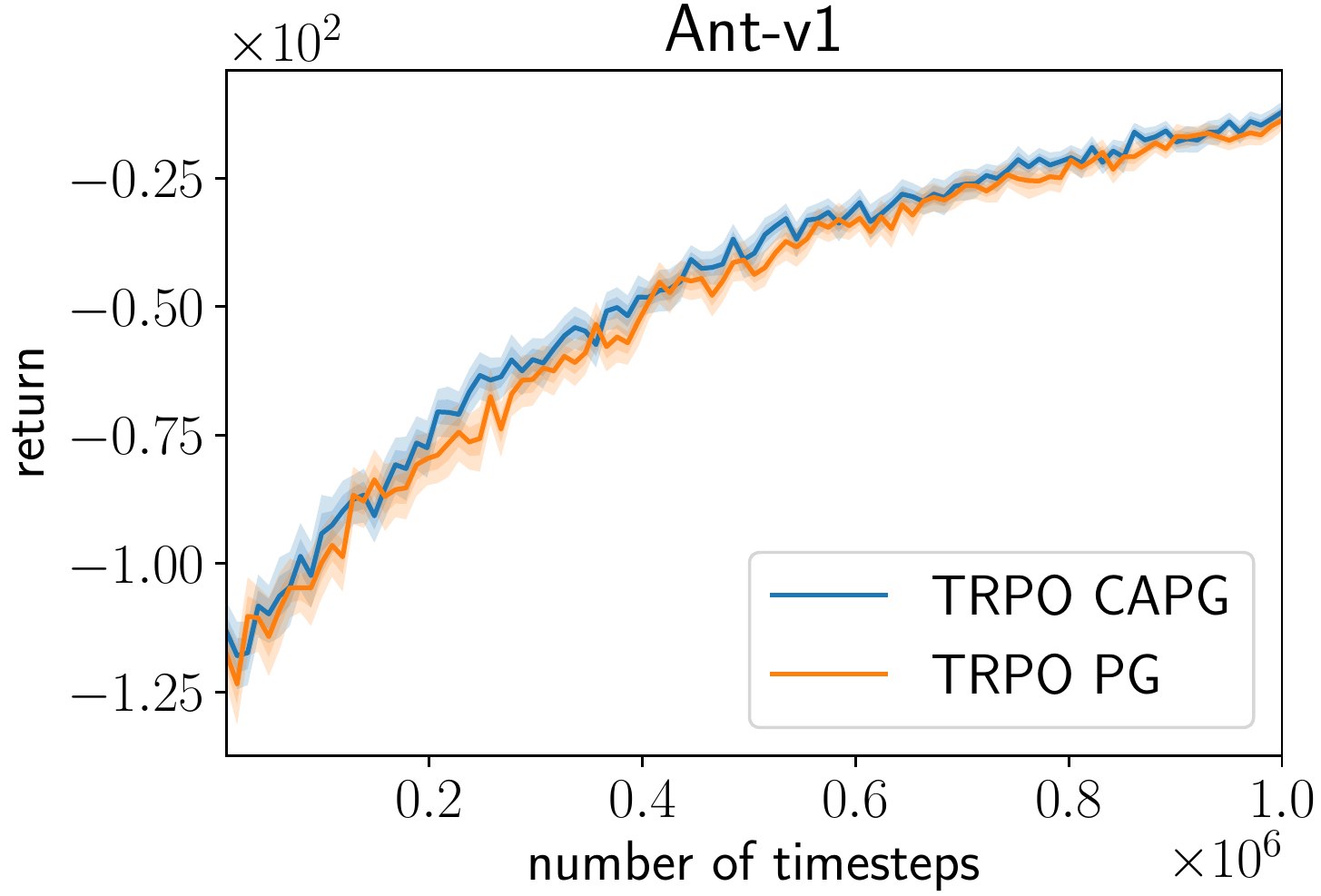}
  \includegraphics[width=0.196\textwidth]{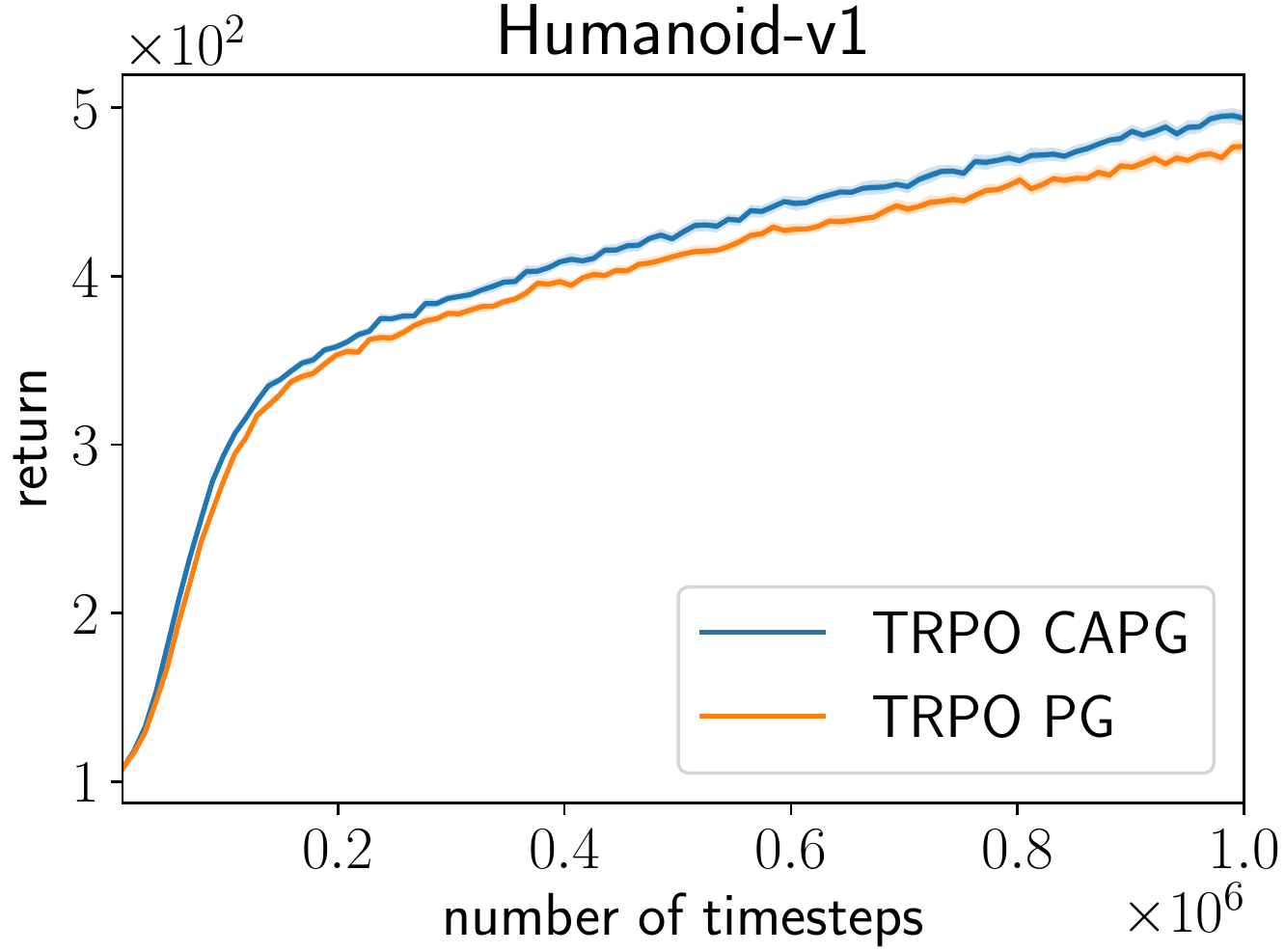}
  \includegraphics[width=0.196\textwidth]{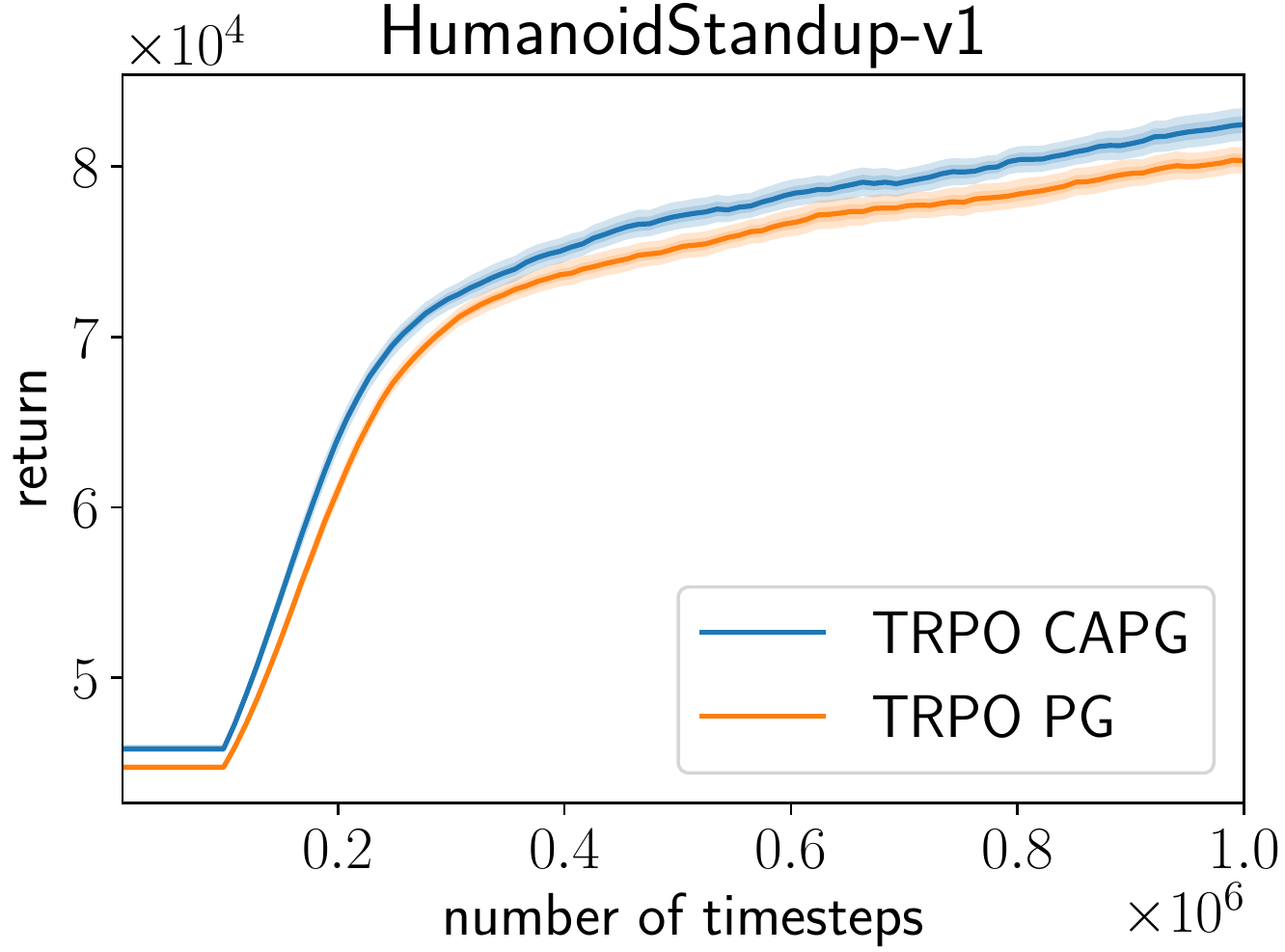}
  \caption{
Training curves of PPO (upper half) and TRPO (lower half) on the 10 MuJoCo-simulated environments.
For each run, after every training episode, the average return of the previous 100 training episodes is computed and linearly interpolated between the episodes to obtain a smoothed curve.
The smoothed curves are then averaged to compute the mean curves with 68\% and 95\% bootstrapped confidence intervals, which are indicated by the shaded areas.
}
  \label{fig:1m_curves}
\end{figure*}
\begin{figure*}[!t]
  \centering
  \includegraphics[width=0.27\textwidth]{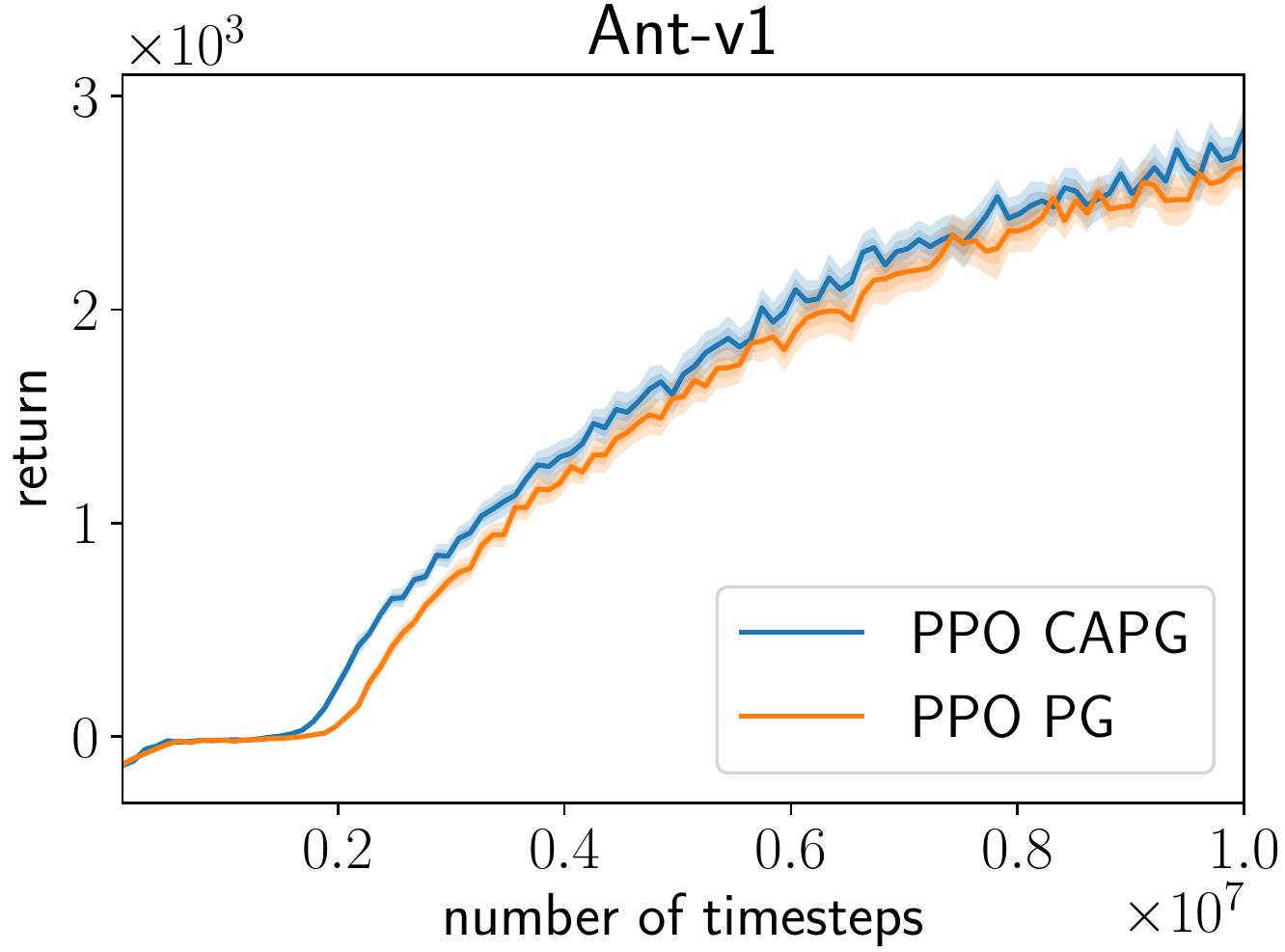}
  \includegraphics[width=0.27\textwidth]{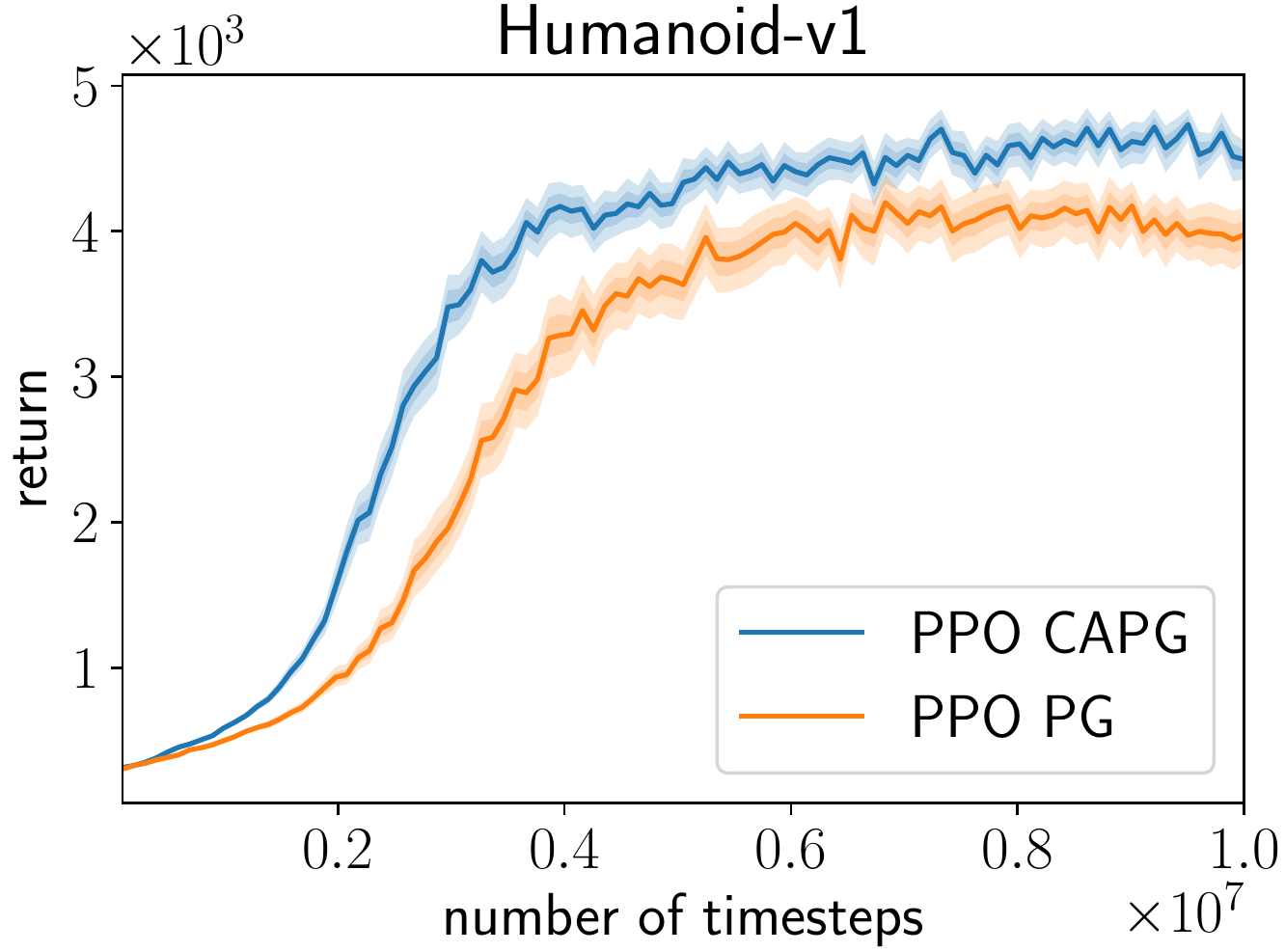}
  \includegraphics[width=0.27\textwidth]{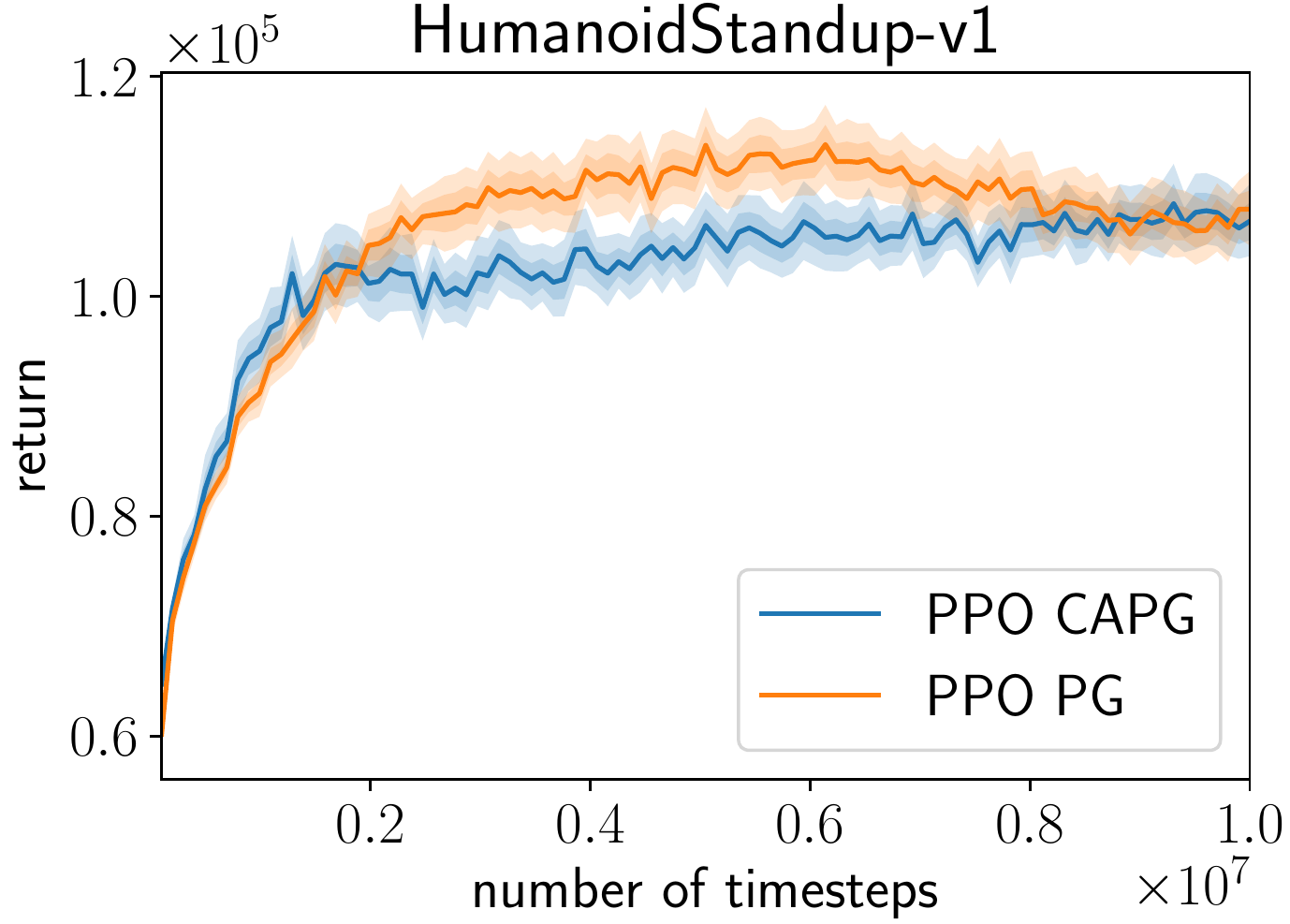}\\
  \includegraphics[width=0.27\textwidth]{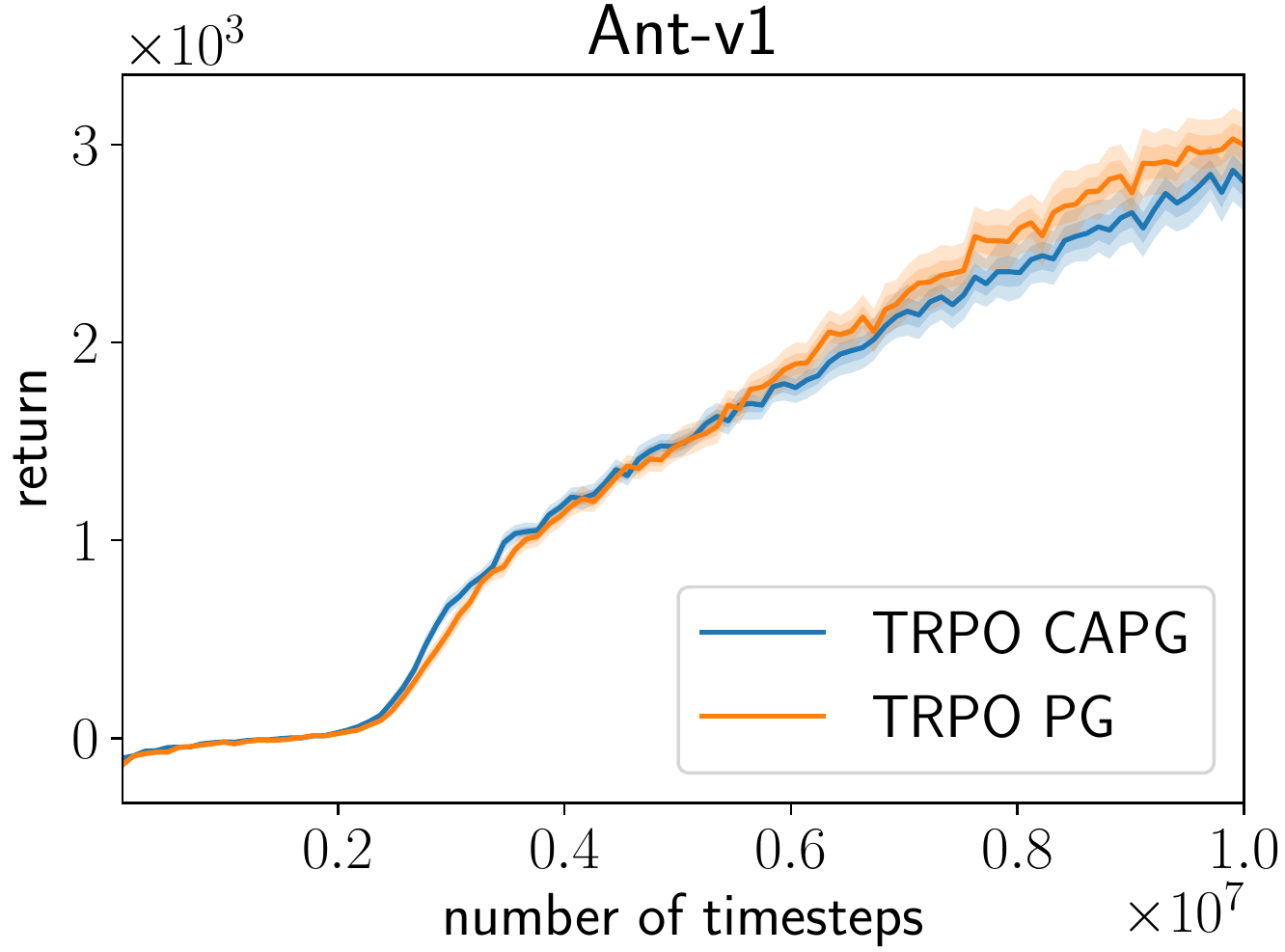}
  \includegraphics[width=0.27\textwidth]{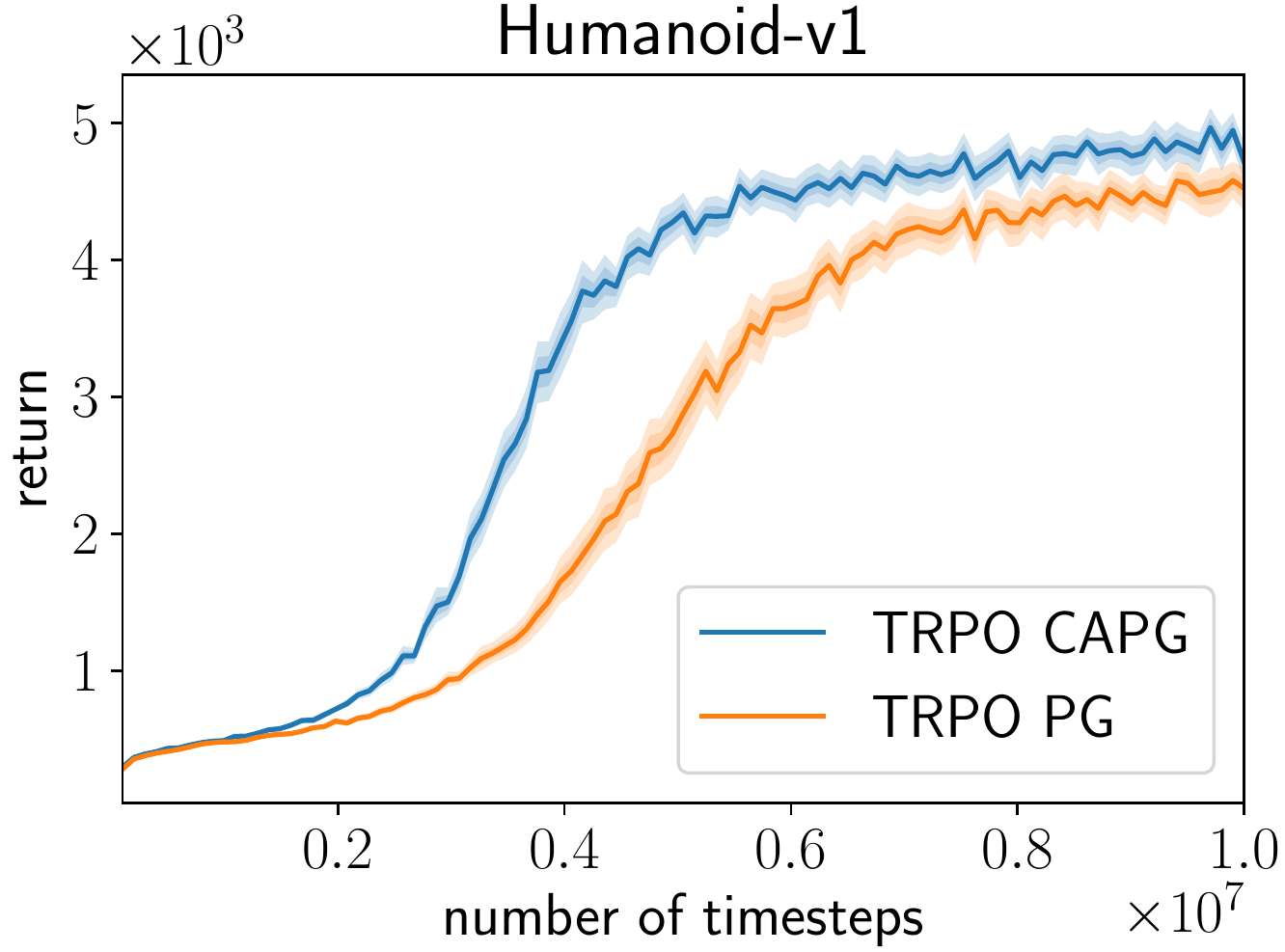}
  \includegraphics[width=0.27\textwidth]{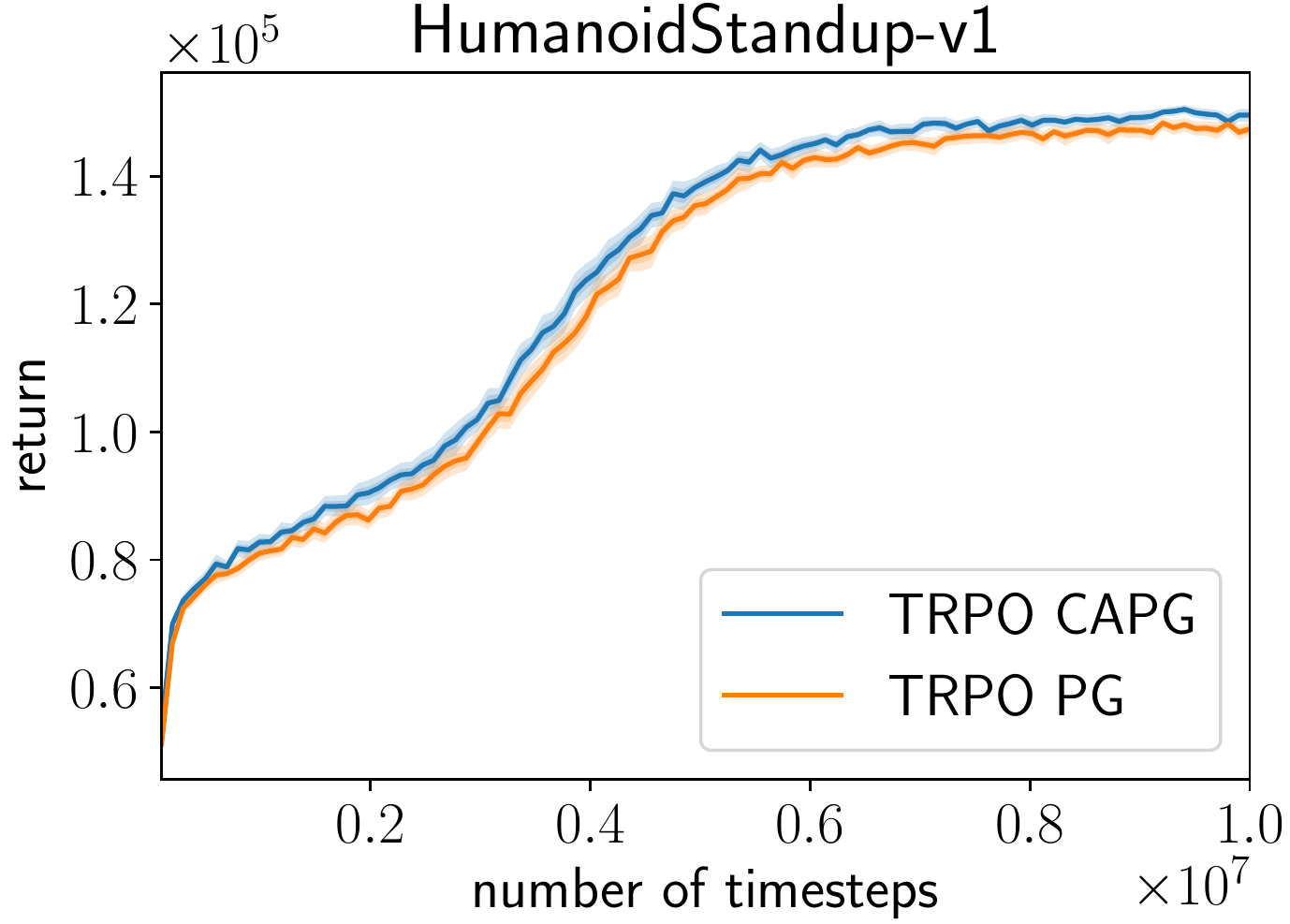}
  \caption{
Training curves of PPO (upper half) and TRPO (lower half) on the three harder MuJoCo-simulated environments.
For each run, after every training episode, the average return of the previous 100 training episodes is computed and linearly interpolated between the episodes to obtain a smoothed curve.
The smoothed curves are then averaged to compute the mean curves with 68\% and 95\% bootstrapped confidence intervals, which are indicated by the shaded areas.
}
  \label{fig:10m_curves}
\end{figure*}

\begin{table*}
  \scriptsize
  \centering
  \begin{tabular}{l|rrr|rrr}
  \toprule
{} & PPO CAPG & PPO PG & $p$-value & TRPO CAPG & TRPO PG & $p$-value \\
\midrule
InvertedPendulum-v1       &               955.30$\pm$1.12 &      955.68$\pm$0.84 &  7.88e-01 &               915.08$\pm$5.23 &      919.94$\pm$0.79 &  3.63e-01 \\
InvertedDoublePendulum-v1 &    \textbf{7239.24$\pm$23.01} &    6991.40$\pm$43.01 &  2.67e-06 &    \textbf{7108.54$\pm$18.17} &    7007.32$\pm$18.95 &  2.07e-04 \\
Reacher-v1                &      \textbf{-10.67$\pm$0.15} &      -11.60$\pm$0.17 &  8.71e-05 &               -14.66$\pm$0.13 &      -14.93$\pm$0.13 &  1.41e-01 \\
Hopper-v1                 &             2320.49$\pm$11.49 &    2288.50$\pm$17.91 &  1.37e-01 &             2313.33$\pm$16.14 &    2283.55$\pm$16.03 &  1.94e-01 \\
HalfCheetah-v1            &             1219.54$\pm$60.94 &    1144.53$\pm$58.81 &  3.78e-01 &              502.05$\pm$18.36 &     499.99$\pm$18.57 &  9.37e-01 \\
Swimmer-v1                &       \textbf{92.56$\pm$3.48} &       82.45$\pm$2.75 &  2.49e-02 &              148.86$\pm$11.44 &     161.18$\pm$11.92 &  4.58e-01 \\
Walker2d-v1               &    \textbf{2185.63$\pm$26.23} &    2060.95$\pm$38.92 &  9.41e-03 &             1436.38$\pm$30.31 &    1390.69$\pm$27.60 &  2.68e-01 \\
Ant-v1                    &       \textbf{56.85$\pm$5.19} &      -33.32$\pm$7.26 &  2.01e-16 &     \textbf{-204.68$\pm$1.84} &     -212.15$\pm$1.92 &  6.04e-03 \\
Humanoid-v1               &      \textbf{547.64$\pm$5.90} &      493.39$\pm$3.89 &  2.49e-11 &      \textbf{415.88$\pm$0.79} &      402.19$\pm$0.75 &  3.96e-22 \\
HumanoidStandup-v1        &  \textbf{79414.10$\pm$496.59} &  76845.37$\pm$512.67 &  5.03e-04 &  \textbf{73592.94$\pm$292.50} &  71796.93$\pm$265.64 &  1.58e-05 \\
\midrule
Ant-v1 (10m)              &    \textbf{1579.54$\pm$10.64} &     1476.51$\pm$15.21 &  2.98e-07 &              1395.50$\pm$28.44 &     1449.61$\pm$32.05 &  2.10e-01 \\
Humanoid-v1 (10m)         &    \textbf{3650.00$\pm$33.98} &     3107.34$\pm$59.01 &  1.06e-11 &     \textbf{3353.08$\pm$23.57} &     2743.53$\pm$40.29 &  1.99e-21 \\
HumanoidStandup-v1 (10m)  &         101826.33$\pm$1012.21 & 105289.56$\pm$1173.48 &  2.78e-02 &  \textbf{123777.22$\pm$383.77} &  120994.09$\pm$403.91 &  2.57e-06 \\
\bottomrule
  \end{tabular}
\caption{
Performance comparison of CAPG and PG on the 10 MuJoCo-simulated environments.
Performance is evaluated with the average area under the learning curve (AUC) $\pm$ standard error over 1 million timesteps.
For each training run, its AUC is computed by linearly interpolating returns between training episodes.
For each combination of \{TRPO, PPO\} $\times$ \{CAPG, PG\} $\times$ 10 environments, from 50 training runs with different random seeds, the average AUC and standard error are computed.
$p$-values are also computed between CAPG and PG versions using Welch's t-test.
Bold numbers indicate that they are better than their counterparts by 95\% significance.
}
\label{tab:auc}
\end{table*}

We considered all the combinations of \{PPO, TRPO\} $\times$ \{CAPG, PG\} $\times$ 10 environments, each of which is trained for 1 million timesteps.
Each combination is tried 50 times with different random seeds.
Because we found it difficult to obtain reasonable performance within 1 million timesteps on Ant-v1, Humanoid-v1, and HumanoidStandup-v1, we also tried training for 10 million timesteps on these environments.

We followed the hyperparameter settings used in \cite{Henderson2017a}, except that the learning rate of Adam used by PPO was reduced to 3e-5 for 10 million timesteps training to obtain reasonable performance with PG.
We used separate neural networks with two hidden layers, each of which has 64 hidden units with tanh nonlinearities, for both a policy and a state value function.
The policy network outputs the mean of a multivariate Gaussian distribution.
The main diagonal of the covariance matrix was separately parameterized as a logarithm of the standard deviation for each dimension.

Table~\ref{tab:auc} summarizes the comparison between CAPG and PG, combined with TRPO and PPO.
We used areas under the learning curves (AUCs) as evaluation measures because they can measure not only the final performance but also the learning speed and stability.

For PPO and TRPO with 1 million training timesteps, CAPG significantly ($p < 0.025$, i.e., $>95\%$ significance) improved AUCs on 3 and 7 out of the 10 environments, respectively.
It also significantly helped in training for 10 million timesteps on two out of the three harder environments for both PPO and TRPO.
On other environments, it kept almost the same level of AUCs on other tasks, although there seemed to be slight decreases in some environments.
These results indicate that CAPG can safely replace PG in many cases.

Figures~\ref{fig:1m_curves} and \ref{fig:10m_curves} show the smoothed learning curves of all the experiments.
In some cases, the improvements were small but consistent, e.g., TRPO on InvertedDoublePendulum-v1 and TRPO on HumanoidStandup-v1 (10 million).
In some other cases, large improvements were achieved, e.g., PPO on Swimmer-v1 and TRPO on Humanoid-v1 (10 million).

Although we used the same hyperparameters from \cite{Henderson2017a} for both PG and CAPG, the best hyperparameters for CAPG can be different.
It is possible that separate hyperparameter tuning can further improve the performance of CAPG.

Comparing the results of PPO and TRPO, PPO was more affected than TRPO by the difference in estimators, suggesting that PPO is more vulnerable to high variance in gradient estimation.
TRPO is likely to be more robust against variance for the following reasons.
\begin{itemize}
  \item TRPO uses a large batch of 5000 actions for every policy update.
    PPO uses minibatches of 64 actions, resulting in noisier updates.
  \item TRPO solves a constrained optimization problem for every policy update so that the change in KL divergence is close to a constant; thus, it is robust to changes in the scale of gradients.
    PPO also adapts its step size using Adam, but this adaptation is slower and based on the statistics of accumulated past gradients.
\end{itemize}
Because we observe that even TRPO can benefit from CAPG, we expect the benefits address other algorithms with noisier updates as well.

\section{Related Work}

A variety of techniques has been proposed to reduce the variance of policy gradient estimation since its introduction.
The control variate method, namely subtracting some baseline from approximate returns, is widely used to reduce the variance while avoiding the introduction of bias into the estimation \cite{Williams1992, Sutton1999, Greensmith2004, Gu2017, Gu2017b}.
Relying on predicted values instead of sampled returns is also popular despite the bias it often introduces \cite{Degris2012, Mnih2016b, Schulman2016, Ciosek2017a}.
Our approach reduces the variance differently from these two common approaches.
Therefore, it can be easily combined with the existing techniques to reduce the variance further while not introducing additional bias.

The problem of using probability distributions with unbounded support for control problems with bounded action spaces was pointed out in \cite{Chou2017}, which proposed modeling policies as beta distributions as a solution.
While they reported performance improvements by using beta policies across multiple continuous control environments, Gaussian policies still nearly dominate the deep RL literature \cite{baselines, Henderson2017a, Tassa2018}.
Truncated distributions have also been used to deal with bounded action spaces in prior work \cite{Nakano2012, Shariff2013, Zimmer2016}.
In contrast, our approach allows us to keep using the same policy parameterizations, typically Gaussians, and still exploit action bounds.
It is also possible to see CAPG as using a multimodal distribution with bounded support, whereas beta policies and truncated Gaussian policies are unimodal.
For example, a clipped Gaussian policy can easily learn to choose end-values of the action bounds with a high probability by moving its mean toward the corresponding end, while beta and truncated Gaussian policies need to be near-deterministic to choose near-end values with a high probability.

Exploiting the integral form of policy gradients to reduce the variance has been proposed in \cite{Ciosek2017a, Asadi2017}.
They directly evaluated the integral over the whole action space, which can be analytically computed for limited classes of action value approximators and policies.
Their method can reduce the variance by eliminating the need for Monte-Carlo estimation of policy gradients while introducing bias from action value approximation.
Our method only evaluates the integral outside the action bounds, i.e., where action values are constant, and thus is unbiased.

\section{Discussion}

We have shown that the variance of policy gradient estimation can be reduced by exploiting the fact that actions are clipped before they are sent to the environment.
An unbiased and lower-variance policy gradient estimator, named CAPG, has been proposed based on our analysis.
CAPG is easy to implement and can be combined with existing variance reduction techniques, such as control variates and value function approximations.

We numerically analyzed CAPG's behavior on simple continuum-armed bandit problems, confirming its efficacy in variance reduction.
When incorporated into existing deep RL algorithms, CAPG generally achieved the same or better performance on challenging simulated control benchmark tasks, indicating its promise as an alternative to the conventional estimator.

While a Gaussian policy is the most common choice in policy gradient-based continuous control, distributions with bounded support may be more suitable for bounded action spaces.
Prior work has proposed beta and truncated distributions to explore this direction.
We argued that CAPG can also be seen as estimating the policy gradient of a transformed distribution with bounded support, termed a clipped distribution.
Further studies are needed on the behaviors of different kinds of distributions as policy representations.

\ifdefined\isaccepted
\section*{Acknowledgments}

We thank Toshiki Kataoka, Kenta Oono, Masaki Watanabe and others at Preferred Networks for insightful comments and discussions.
\fi

\bibliography{ref}

\begin{thebibliography}{29}
\providecommand{\natexlab}[1]{#1}
\providecommand{\url}[1]{\texttt{#1}}
\expandafter\ifx\csname urlstyle\endcsname\relax
  \providecommand{\doi}[1]{doi: #1}\else
  \providecommand{\doi}{doi: \begingroup \urlstyle{rm}\Url}\fi

\bibitem[Agrawal(1995)]{Agrawal1995}
Agrawal, R.
\newblock {The Continuum-Armed Bandit Problem}.
\newblock \emph{SIAM Journal on Control and Optimization}, 33\penalty0
  (6):\penalty0 1926--1951, 1995.

\bibitem[Asadi et~al.(2017)Asadi, Allen, Roderick, Mohamed, Konidaris, and
  Littman]{Asadi2017}
Asadi, K., Allen, C., Roderick, M., Mohamed, A.-r., Konidaris, G., and Littman,
  M.
\newblock {Mean Actor Critic}.
\newblock \emph{ArXiv e-prints}, 2017.

\bibitem[Brockman et~al.(2016)Brockman, Cheung, Pettersson, Schneider,
  Schulman, Tang, and Zaremba]{Brockman2016}
Brockman, G., Cheung, V., Pettersson, L., Schneider, J., Schulman, J., Tang,
  J., and Zaremba, W.
\newblock {OpenAI Gym}.
\newblock \emph{ArXiv e-prints}, 2016.

\bibitem[Chou et~al.(2017)Chou, Maturana, and Scherer]{Chou2017}
Chou, P.-W., Maturana, D., and Scherer, S.
\newblock {Improving Stochastic Policy Gradients in Continuous Control with
  Deep Reinforcement Learning using the Beta Distribution}.
\newblock In \emph{ICML}, 2017.

\bibitem[Ciosek \& Whiteson(2018)Ciosek and Whiteson]{Ciosek2017a}
Ciosek, K. and Whiteson, S.
\newblock {Expected Policy Gradients}.
\newblock In \emph{AAAI}, 2018.

\bibitem[Degris et~al.(2012)Degris, White, and Sutton]{Degris2012}
Degris, T., White, M., and Sutton, R.~S.
\newblock {Off-Policy Actor-Critic}.
\newblock In \emph{ICML}, 2012.

\bibitem[Dhariwal et~al.(2017)Dhariwal, Hesse, Klimov, Nichol, Plappert,
  Radford, Schulman, Sidor, and Wu]{baselines}
Dhariwal, P., Hesse, C., Klimov, O., Nichol, A., Plappert, M., Radford, A.,
  Schulman, J., Sidor, S., and Wu, Y.
\newblock {OpenAI Baselines}.
\newblock \url{https://github.com/openai/baselines}, 2017.

\bibitem[Duan et~al.(2016)Duan, Chen, Houthooft, Schulman, and
  Abbeel]{Duan2016}
Duan, Y., Chen, X., Houthooft, R., Schulman, J., and Abbeel, P.
\newblock {Benchmarking Deep Reinforcement Learning for Continuous Control}.
\newblock In \emph{ICML}, 2016.

\bibitem[Greensmith et~al.(2004)Greensmith, Bartlett, and
  Baxter]{Greensmith2004}
Greensmith, E., Bartlett, P., and Baxter, J.
\newblock {Variance Reduction Techniques for Gradient Estimates in
  Reinforcement Learning}.
\newblock \emph{The Journal of Machine Learning Research}, 5:\penalty0
  1471--1530, 2004.

\bibitem[Gu et~al.(2017{\natexlab{a}})Gu, Lillicrap, Ghahramani, Turner, and
  Levine]{Gu2017}
Gu, S., Lillicrap, T., Ghahramani, Z., Turner, R.~E., and Levine, S.
\newblock {Q-Prop: Sample-Efficient Policy Gradient with an Off-Policy Critic}.
\newblock In \emph{ICLR}, 2017{\natexlab{a}}.

\bibitem[Gu et~al.(2017{\natexlab{b}})Gu, Lillicrap, Ghahramani, Turner,
  Sch{\"{o}}lkopf, and Levine]{Gu2017b}
Gu, S., Lillicrap, T., Ghahramani, Z., Turner, R.~E., Sch{\"{o}}lkopf, B., and
  Levine, S.
\newblock {Interpolated Policy Gradient : Merging On-Policy and Off-Policy
  Gradient Estimation for Deep}.
\newblock In \emph{NIPS}, 2017{\natexlab{b}}.

\bibitem[Heess et~al.(2017)Heess, TB, Sriram, Lemmon, Merel, Wayne, Tassa,
  Erez, Wang, Eslami, Riedmiller, and Silver]{Heess2017}
Heess, N., TB, D., Sriram, S., Lemmon, J., Merel, J., Wayne, G., Tassa, Y.,
  Erez, T., Wang, Z., Eslami, S. M.~A., Riedmiller, M., and Silver, D.
\newblock {Emergence of Locomotion Behaviours in Rich Environments}.
\newblock \emph{ArXiv e-prints}, 2017.

\bibitem[Henderson et~al.(2018)Henderson, Islam, Bachman, Pineau, Precup, and
  Meger]{Henderson2017a}
Henderson, P., Islam, R., Bachman, P., Pineau, J., Precup, D., and Meger, D.
\newblock {Deep Reinforcement Learning that Matters}.
\newblock In \emph{AAAI}, 2018.

\bibitem[Kingma \& Ba(2015)Kingma and Ba]{Kingma2015b}
Kingma, D.~P. and Ba, J.~L.
\newblock {Adam: a Method for Stochastic Optimization}.
\newblock In \emph{ICLR}, 2015.

\bibitem[Levine et~al.(2016)Levine, Finn, Darrell, and Abbeel]{Levine2016}
Levine, S., Finn, C., Darrell, T., and Abbeel, P.
\newblock {End-to-End Training of Deep Visuomotor Policies}.
\newblock \emph{The Journal of Machine Learning Research}, 17\penalty0
  (1):\penalty0 1334--1373, 2016.

\bibitem[Mnih et~al.(2015)Mnih, Kavukcuoglu, Silver, Rusu, Veness, Bellemare,
  Graves, Riedmiller, Fidjeland, Ostrovski, Petersen, Beattie, Sadik,
  Antonoglou, King, Kumaran, Wierstra, Legg, and Hassabis]{Mnih2015}
Mnih, V., Kavukcuoglu, K., Silver, D., Rusu, A.~a., Veness, J., Bellemare,
  M.~G., Graves, A., Riedmiller, M., Fidjeland, A.~K., Ostrovski, G., Petersen,
  S., Beattie, C., Sadik, A., Antonoglou, I., King, H., Kumaran, D., Wierstra,
  D., Legg, S., and Hassabis, D.
\newblock {Human-level control through deep reinforcement learning}.
\newblock \emph{Nature}, 518\penalty0 (7540):\penalty0 529--533, 2015.

\bibitem[Mnih et~al.(2016)Mnih, Badia, Mirza, Graves, Lillicrap, Harley,
  Silver, and Kavukcuoglu]{Mnih2016b}
Mnih, V., Badia, A.~P., Mirza, M., Graves, A., Lillicrap, T.~P., Harley, T.,
  Silver, D., and Kavukcuoglu, K.
\newblock {Asynchronous Methods for Deep Reinforcement Learning}.
\newblock In \emph{ICML}, 2016.

\bibitem[Nakano et~al.(2012)Nakano, Maeda, and Ishii]{Nakano2012}
Nakano, D., Maeda, S.-i., and Ishii, S.
\newblock {Control of a Free-Falling Cat by Policy-Based Reinforcement
  Learning}.
\newblock In \emph{ICANN}, 2012.

\bibitem[Schulman et~al.(2015)Schulman, Levine, Moritz, Jordan, and
  Abbeel]{Schulman2015e}
Schulman, J., Levine, S., Moritz, P., Jordan, M., and Abbeel, P.
\newblock {Trust Region Policy Optimization}.
\newblock In \emph{ICML}, 2015.

\bibitem[Schulman et~al.(2016)Schulman, Moritz, Levine, Jordan, and
  Abbeel]{Schulman2016}
Schulman, J., Moritz, P., Levine, S., Jordan, M.~I., and Abbeel, P.
\newblock {High-Dimensional Continuous Control Using Generalized Advantage
  Estimation}.
\newblock In \emph{ICLR}, 2016.

\bibitem[Schulman et~al.(2017)Schulman, Wolski, Dhariwal, Radford, and
  Klimov]{Schulman2017b}
Schulman, J., Wolski, F., Dhariwal, P., Radford, A., and Klimov, O.
\newblock {Proximal Policy Optimization Algorithms}.
\newblock \emph{ArXiv e-prints}, 2017.

\bibitem[Shariff \& Dick(2013)Shariff and Dick]{Shariff2013}
Shariff, R. and Dick, T.
\newblock {Lunar Lander : A Continous-Action Case Study for Policy-Gradient
  Actor-Critic Algorithms}.
\newblock In \emph{RLDM}, 2013.

\bibitem[Silver et~al.(2016)Silver, Huang, Maddison, Guez, Sifre, Driessche,
  Schrittwieser, Antonoglou, Panneershelvam, Lanctot, Dieleman, Grewe, Nham,
  Kalchbrenner, Sutskever, Lillicrap, Leach, and Kavukcuoglu]{Silver2016a}
Silver, D., Huang, A., Maddison, C.~J., Guez, A., Sifre, L., Driessche, G.
  V.~D., Schrittwieser, J., Antonoglou, I., Panneershelvam, V., Lanctot, M.,
  Dieleman, S., Grewe, D., Nham, J., Kalchbrenner, N., Sutskever, I.,
  Lillicrap, T., Leach, M., and Kavukcuoglu, K.
\newblock {Mastering the game of Go with deep neural networks and tree search}.
\newblock \emph{Nature}, 529\penalty0 (7585):\penalty0 484--489, 2016.

\bibitem[Silver et~al.(2017)Silver, Schrittwieser, Simonyan, Antonoglou, Huang,
  Guez, Hubert, Baker, Lai, Bolton, Chen, Lillicrap, Hui, and
  Sifre]{Silver2017}
Silver, D., Schrittwieser, J., Simonyan, K., Antonoglou, I., Huang, A., Guez,
  A., Hubert, T., Baker, L., Lai, M., Bolton, A., Chen, Y., Lillicrap, T., Hui,
  F., and Sifre, L.
\newblock {Mastering the game of Go without human knowledge}.
\newblock \emph{Nature Publishing Group}, 550\penalty0 (7676):\penalty0
  354--359, 2017.

\bibitem[Sutton et~al.(1999)Sutton, Mcallester, Singh, and Mansour]{Sutton1999}
Sutton, R.~S., Mcallester, D., Singh, S., and Mansour, Y.
\newblock {Policy Gradient Methods for Reinforcement Learning with Function
  Approximation}.
\newblock In \emph{NIPS}, 1999.

\bibitem[Tassa et~al.(2018)Tassa, Doron, Muldal, Erez, Li, De, Casas, Budden,
  Abdolmaleki, Merel, Lefrancq, Lillicrap, and Riedmiller]{Tassa2018}
Tassa, Y., Doron, Y., Muldal, A., Erez, T., Li, Y., De, D., Casas, L., Budden,
  D., Abdolmaleki, A., Merel, J., Lefrancq, A., Lillicrap, T., and Riedmiller,
  M.
\newblock {DeepMind Control Suite}.
\newblock \emph{ArXiv e-prints}, 2018.

\bibitem[Todorov et~al.(2012)Todorov, Erez, and Tassa]{Todorov2012}
Todorov, E., Erez, T., and Tassa, Y.
\newblock {MuJoCo: A physics engine for model-based control}.
\newblock In \emph{IROS}, 2012.

\bibitem[Williams(1992)]{Williams1992}
Williams, R.
\newblock {Simple Statistical Gradient-Following Algorithms for Connectionist
  Reinforcement Learning}.
\newblock \emph{Machine Learning}, 8\penalty0 (3-4):\penalty0 229--256, 1992.

\bibitem[Zimmer et~al.(2016)Zimmer, Boniface, and Dutech]{Zimmer2016}
Zimmer, M., Boniface, Y., and Dutech, A.
\newblock {Off-policy Neural Fitted Actor-Critic}.
\newblock In \emph{NIPS Deep Reinforcement Learning Workshop}, 2016.

\end{thebibliography}
\bibliographystyle{icml2018}

\clearpage

\onecolumn

\section*{Appendix}

\lemmaestimator*
\begin{proof}
Noting that $\pi_\theta(u|s)$ allows the exchange of derivative and integral, we get
\begin{align}
\E_u[\I_{u \le \low} \nabla_\theta \log \pi_\theta(u|s)] \nonumber
  &= \int_{-\infty}^\low \pi_\theta(u|s) \nabla_\theta \log \pi_\theta(u|s)du \nonumber
\\&= \int_{-\infty}^\low \nabla_\theta \pi_\theta(u|s)du \nonumber
\\&= \nabla_\theta \int_{-\infty}^\low \pi_\theta(u|s)du \nonumber
\\&= \nabla_\theta \Pi_\theta(\low|s) \nonumber
\\&= \Pi_\theta(\low|s) \nabla_\theta \log \Pi_\theta(\low|s) \nonumber
\\&= \E_u[\I_{u \le \low} \nabla_\theta \log \Pi_\theta(\low|s)]. \nonumber
\end{align}

A similar calculation shows
\begin{equation}
\E_u[\I_{\high \le u} \nabla_\theta \log \pi_\theta(u|s)] \nonumber
= \E_u [\I_{\high \le u} \nabla_\theta \log (1 - \Pi_\theta(\high|s))],
\end{equation}
where we used $\int_{\high}^{\infty} \pi_\theta(u|s)du = 1 - \Pi_\theta(\high|s)$
instead of $\int_{-\infty}^\low \pi_\theta(u|s)du = \Pi_\theta(\low|s) $.
\end{proof}

\lemmavarianceinequality*
\begin{proof}
Because both $\I_{u \le \low} \nabla_\theta \log \pi_\theta(u|s)$ and $\I_{u \le \low} \nabla_\theta \log \Pi_\theta(\low|s)$ have the same expected values from Lemma~\ref{lemma:estimator}, the difference of their variances is written as follows:
\begin{align}
\Var_u[\I_{u \le \low} \nabla_\theta \log \pi_\theta(u|s)] - \Var_u[\I_{u \le \low} \nabla_\theta \log \Pi_\theta(\low|s)] = \E_u[\I_{u \le \low} (\nabla_\theta \log \pi_\theta(u|s))^2] - \E_u[\I_{u \le \low} (\nabla_\theta \log \Pi_\theta(\low|s))^2].
\end{align}
The difference above is nonnegative because
\begin{align}
\E_u[\I_{u \le \low} (\nabla_\theta \log \pi_\theta(u|s))^2] \nonumber
  &=\int_{-\infty}^{\low} \pi_\theta(u|s) (\nabla_\theta \log \pi_\theta(u|s))^2du \nonumber
\\&= \Pi_\theta(\low|s) \int_{-\infty}^{\infty} \I_{u \le \low}\frac{\pi_\theta(u|s)}{\Pi_\theta(\low|s)} (\nabla_\theta \log \pi_\theta(u|s))^2du \nonumber
\\&\ge \Pi_\theta(\low|s) \Big(\int_{-\infty}^{\infty} \I_{u \le \low}\frac{\pi_\theta(u|s)}{\Pi_\theta(\low|s)} \nabla_\theta \log \pi_\theta(u|s)du \Big)^2 \nonumber
\\&= \Pi_\theta(\low|s) \Big(\frac{1}{\Pi_\theta(\low|s)} \int_{-\infty}^{\low} \pi_\theta(u|s) \nabla_\theta \log \pi_\theta(u|s)du \Big)^2 \nonumber
\\&= \Pi_\theta(\low|s) \Big(\frac{1}{\Pi_\theta(\low|s)} \int_{-\infty}^{\low} \nabla_\theta \pi_\theta(u|s)du\Big)^2 \nonumber
\\&= \Pi_\theta(\low|s) \Big(\frac{1}{\Pi_\theta(\low|s)} \nabla_\theta \int_{-\infty}^{\low} \pi_\theta(u|s)du\Big)^2 \nonumber
\\&= \Pi_\theta(\low|s) \Big(\frac{1}{\Pi_\theta(\low|s)} \nabla_\theta \Pi_\theta(\low|s)\Big)^2 \nonumber
\\&= \Pi_\theta(\low|s) (\nabla_\theta \log \Pi_\theta(\low|s))^2 \nonumber
\\&= \E_u[\I_{u \le \low} (\nabla_\theta \log \Pi_\theta(\low|s))^2], \nonumber
\end{align}
where the equality holds only when $\nabla_\theta \log \pi_\theta(u|s)$ is constant over $u \le \low$.

A similar calculation shows
\begin{equation}
  \Var_u[\I_{\high \le u} \nabla_\theta \log \pi_\theta(u|s)] \geq \Var_u[\I_{\high \le u} \nabla_\theta \log (1-\Pi_\theta(\high|s))],
\end{equation}
where the equality holds only when $\nabla_\theta \log \pi_\theta(u|s)$ is constant over $\high \le u$.
\end{proof}

\lemmavectoraction*
\begin{proof}

Applying Lemma~\ref{lemma:CAPG} to each $u_i$ yields
\begin{align}
  \E_{\vec{u}}[f(s,\vec{u})\capgpsi^{(i)}(s,u_i)] &= \E_{\vec{u}}[f(s,\vec{u})\psi^{(i)}(s,u_i)],\label{eq:ui_expectation_equality} \\
  \V_{\vec{u}}[f(s,\vec{u})\capgpsi^{(i)}(s,u_i)] &\le \V_{\vec{u}}[f(s,\vec{u})\psi^{(i)}(s,u_i)].\label{eq:ui_variance_inequality}
\end{align}

Because each action is conditionally independent, we can decompose the expectations as
\begin{align}
  \E_{\vec{u}}[f(s,\vec{u})\psi(s,\vec{u})]     &= \sum_{i=1}^d \E_{\vec{u}}[f(s,\vec{u})\psi^{(i)}(s,u_i)], \\
  \E_{\vec{u}}[f(s,\vec{u})\capgpsi(s,\vec{u})] &= \sum_{i=1}^d \E_{\vec{u}}[f(s,\vec{u})\capgpsi^{(i)}(s,u_i)].
\end{align}
From \eqref{eq:ui_expectation_equality}, these two are equal, and hence \eqref{eq:vector_capg_equality} holds.

The variances can also be decomposed as
\begin{align}
  \V_{\vec{u}}[f(s,\vec{u})\psi(s,\vec{u})] &= \sum_{i=1}^d \V_{\vec{u}}[f(s,\vec{u})\psi^{(i)}(s,u_i)] + \sum_{1 \le i < j \le d} \Cov[f(s,\vec{u})\psi^{(i)}(s,u_i),f(s,\vec{u})\psi^{(j)}(s,u_j)], \label{eq:var_vector_psi} \\
  \V_{\vec{u}}[f(s,\vec{u})\capgpsi(s,\vec{u})] &= \sum_{i=1}^d \V_{\vec{u}}[f(s,\vec{u})\capgpsi^{(i)}(s,u_i)] + \sum_{1 \le i < j \le d} \Cov[f(s,\vec{u})\capgpsi^{(i)}(s,u_i),f(s,\vec{u})\capgpsi^{(j)}(s,u_j)], \label{eq:var_vector_capgpsi}
\end{align}
where
\begin{align}
\MoveEqLeft
  \Cov[f(s,\vec{u})\psi^{(i)}(s,u_i),f(s,\vec{u})\psi^{(j)}(s,u_j)]\\
  &= \E_{\vec{u}}[(f(s,\vec{u}))^2\psi^{(i)}(s,u_i)\psi^{(j)}(s,u_j)] - \E_{\vec{u}}[f(s,\vec{u})\psi^{(i)}(s,u_i)]\E_{\vec{u}}[f(s,\vec{u})\psi^{(j)}(s,u_j)], \label{eq:cov_psi}\\
\MoveEqLeft
  \Cov[f(s,\vec{u})\capgpsi^{(i)}(s,u_i),f(s,\vec{u})\capgpsi^{(j)}(s,u_j)]\\
  &= \E_{\vec{u}}[(f(s,\vec{u}))^2\capgpsi^{(i)}(s,u_i)\capgpsi^{(j)}(s,u_j)] - \E_{\vec{u}}[f(s,\vec{u})\capgpsi^{(i)}(s,u_i)]\E_{\vec{u}}[f(s,\vec{u})\capgpsi^{(j)}(s,u_j)]. \label{eq:cov_capgpsi}
\end{align}
The first term of \eqref{eq:var_vector_capgpsi} is smaller than or equal to that of \eqref{eq:var_vector_psi} from \eqref{eq:ui_variance_inequality}.
Thus, to prove \eqref{eq:vector_capg_inequality}, it is sufficient to show that the second terms of \eqref{eq:var_vector_psi} and \eqref{eq:var_vector_capgpsi} are equal.

The second terms of \eqref{eq:cov_psi} and \eqref{eq:cov_capgpsi} are equal from \eqref{eq:ui_expectation_equality}.
Using the law of total variance, the first terms of \eqref{eq:cov_psi} and \eqref{eq:cov_capgpsi} can be written as
\begin{align}
\E_{\vec{u}}[(f(s,\vec{u}))^2\psi^{(i)}(s,u_i)\psi^{(j)}(s,u_j)]
&= \E_{\vec{u}_{\setminus i,j}}\bigg[\E_{u_i}\Big[\E_{u_j}[(f(s,\vec{u}))^2\psi^{(i)}(s,u_i)\psi^{(j)}(s,u_j)|\vec{u}_{\setminus j}]\Big|\vec{u}_{\setminus i,j}\Big]\bigg]\\
&= \E_{\vec{u}_{\setminus i,j}}\bigg[\E_{u_i}\Big[\psi^{(i)}(s,u_i)\E_{u_j}[(f(s,\vec{u}))^2\psi^{(j)}(s,u_j)|\vec{u}_{\setminus j}]\Big|\vec{u}_{\setminus i,j}\Big]\bigg], \label{eq:cov_first_psi}\\
\E_{\vec{u}}[(f(s,\vec{u}))^2\capgpsi^{(i)}(s,u_i)\capgpsi^{(j)}(s,u_j)] &= \E_{\vec{u}_{\setminus i,j}}\bigg[\E_{u_i}\Big[\capgpsi^{(i)}(s,u_i)\E_{u_j}[(f(s,\vec{u}))^2\capgpsi^{(j)}(s,u_j)|\vec{u}_{\setminus j}]\Big|\vec{u}_{\setminus i,j}\Big]\bigg] \label{eq:cov_first_capgpsi},
\end{align}
where $\vec{u}_{\setminus j}$ denotes a vector $\vec{u}$ with the $j$-th element excluded, and $\vec{u}_{\setminus i,j}$ denotes a vector $\vec{u}$ with the $i$-th and $j$-th elements excluded.
Noting the fact that $(f(s,\vec{u}))^2$ is a function of $u_j$ conditioned on $s$ and $\vec{u}_{\setminus j}$, we can have the following equation by applying Lemma~\ref{lemma:CAPG}.
\begin{equation}
  \E_{u_j}[(f(s,\vec{u}))^2\capgpsi^{(j)}(s,u_j)|\vec{u}_{\setminus j}] = \E_{u_j}[(f(s,\vec{u}))^2\psi^{(j)}(s,u_j)|\vec{u}_{\setminus j}].
\end{equation}
Similarly, we can use the fact that $\E_{u_j}[(f(s,\vec{u}))^2\psi^{(j)}(s,u_j)|\vec{u}_{\setminus j}]$ is a function of $u_i$ conditioned on $s$ and $\vec{u}_{\setminus i,j}$ to show
\begin{equation}
  \E_{u_i}\Big[\capgpsi^{(i)}(s,u_i)\E_{u_j}[(f(s,\vec{u}))^2\capgpsi^{(j)}(s,u_j)|\vec{u}_{\setminus j}]\Big|\vec{u}_{\setminus i,j}\Big] = \E_{u_i}\Big[\psi^{(i)}(s,u_i)\E_{u_j}[(f(s,\vec{u}))^2\psi^{(j)}(s,u_j)|\vec{u}_{\setminus j}]\Big|\vec{u}_{\setminus i,j}\Big].
  \label{eq:inner_expectation_eq}
\end{equation}
From \eqref{eq:inner_expectation_eq}, we can see \eqref{eq:cov_first_psi} and \eqref{eq:cov_first_capgpsi} are equal.
This implies that the first terms of \eqref{eq:cov_psi} and \eqref{eq:cov_capgpsi} are equal, and the second terms of \eqref{eq:var_vector_psi} and \eqref{eq:var_vector_capgpsi} are equal.
Therefore, \eqref{eq:vector_capg_inequality} is satisfied.
The equality of \eqref{eq:vector_capg_inequality} holds only when $\nabla_\theta \log \pitheta^{(i)}(u|s)$ is constant over both $u \le \low$ and $\high \le u$ for all $1 \le i \le d$.
\end{proof}

\end{document}